\documentclass[11pt]{article}
\usepackage[T1]{fontenc}
\usepackage[latin9]{inputenc}
\usepackage{amsmath}
\usepackage{amsthm}
\usepackage{amssymb}
\usepackage{graphicx}
\usepackage{graphics}
\usepackage[caption=false]{subfig}
\usepackage{hyperref}
\usepackage{enumerate}
\usepackage[linesnumbered,ruled,vlined]{algorithm2e}
\usepackage[noend]{algpseudocode}

\makeatletter

\newcommand{\opnorm}{\@ifstar\@opnorms\@opnorm}
\newcommand{\@opnorms}[1]{%
	\left|\mkern-1.5mu\left|\mkern-1.5mu\left|
	#1
	\right|\mkern-1.5mu\right|\mkern-1.5mu\right|
}
\newcommand{\@opnorm}[2][]{%
	\mathopen{#1|\mkern-1.5mu#1|\mkern-1.5mu#1|}
	#2
	\mathclose{#1|\mkern-1.5mu#1|\mkern-1.5mu#1|}
}
\makeatother

\newtheorem{theorem}{Theorem}
\newtheorem{definition}{Definition}
\newtheorem{lemma}{Lemma}

\newtheorem{remark}{Remark}
\newtheorem{corollary}{Corollary}
\newtheorem{proposition}{Proposition}
\newtheorem{example}{Example}
\newtheorem{assumption}{Assumption}

\usepackage[margin=1in]{geometry}
\usepackage{times}

\usepackage{thmtools}
\usepackage{thm-restate}

\usepackage{listings}

\global\long\def\bx{\mathbf{x}}
\global\long\def\bd{\mathbf{d}}
\global\long\def\bw{\mathbf{w}}
\global\long\def\bu{\mathbf{u}}
\global\long\def\bv{\mathbf{v}}

\global\long\def\by{\mathbf{y}}

\global\long\def\mG{\mathcal{G}}
\global\long\def\mB{\mathcal{B}}
\global\long\def\deg{\mathrm{deg}}

\begin{document}
	
	\title{Exact Guarantees on the Absence of Spurious Local Minima for Non-negative {Rank-1} Robust Principal Component Analysis}
	\author{
		Salar Fattahi and Somayeh Sojoudi
		\thanks{Email: fattahi@berkeley.edu and sojoudi@berkeley.edu.
			Salar Fattahi is with the Department of Industrial Engineering and Operations Research, University of California, Berkeley. Somayeh Sojoudi is with the Departments of Electrical Engineering and Computer Sciences and Mechanical Engineering as well as the Tsinghua-Berkeley Shenzhen Institute, University of California, Berkeley. This work was supported by the ONR Award N00014-18-1-2526 
			and an NSF EPCN Grant. }}
	\date{}
	\maketitle

	\begin{abstract}
		This work is concerned with the non-negative {rank-1} robust principal component analysis (RPCA), where the goal is to recover the dominant non-negative principal components of a data matrix \textit{precisely}, where a number of measurements could be grossly corrupted with sparse and arbitrary large noise. Most of the known techniques for solving the RPCA rely on convex relaxation methods by lifting the problem to a higher dimension, which significantly increase the number of variables. As an alternative, the well-known Burer-Monteiro approach can be used to cast the RPCA as a non-convex and non-smooth $\ell_1$ optimization problem with a significantly smaller number of variables. In this work, we show that the low-dimensional formulation of the symmetric and asymmetric positive {rank-1} RPCA based on the Burer-Monteiro approach has benign landscape, i.e., 1) it does not have any spurious local solution, 2) has a unique global solution, and 3) its unique global solution coincides with the \textit{true} components. 
		An implication of this result is that simple local search algorithms are guaranteed to achieve a zero global optimality gap when directly applied to the low-dimensional formulation. 
		Furthermore, we provide strong deterministic and probabilistic guarantees for the exact recovery of the true principal components. In particular, it is shown that a constant fraction of the measurements could be grossly corrupted and yet they would not create any spurious local solution.
	\end{abstract}
	
	\section{Introduction}
	\begin{sloppypar}
		The principal component analysis (PCA) is perhaps the most widely-used dimension-reduction method that reveals the components with maximum variability in high-dimensional datasets. In particular, given the data matrix $X\in\mathbb{R}^{m\times n}$, where each row corresponds to a data sample with size $n$, the goal is to recover its most dominant component under the {rank-1} spiked model\footnote[1]{There are more general models under which the PCA is shown to be useful (see~\cite{jolliffe2011principal} for more details). We use the {rank-1} spiked model since it fits into our framework and is often used as a baseline to evaluate the performance of the PCA.}
	\end{sloppypar}
	\begin{equation}\label{spike}
	X = \beta\bu\bv^\top+S
	\end{equation}
	where $\beta$ determines the signal-to-noise ratio, $S$ is the additive noise matrix, and $\bu$ and $\bv$ are two unknown unit norm vectors. If the data matrix $X$ is symmetric (for instance, it corresponds to a sample covariance matrix), then~\eqref{spike} can be modified as
	\begin{equation}
	X = \beta\bv\bv^\top+S
	\end{equation}
	Depending on the nature of the noise matrix, different methods have been proposed in the literature to recover the principal components from (partial) observations of $X$. The problem of recovering $\beta$, $\bu$, and $\bv$ under a Gaussian and sparse noise is conventionally referred to as PCA and robust PCA (or RPCA), respectively.
	
	The properties of both PCA and its robust analog have been heavily studied in the literature and their applications span from quantitative finance to health care and neuroscience ~(\cite{hull1990pricing, caprihan2008application, brenner2000adaptive}). Recently, a special focus has been devoted to further exploiting the prior knowledge on the principal components, such as sparsity~(\cite{zou2006sparse}) and nonlinearity~(\cite{gorban2008principal}). Accordingly, one such knowledge appearing in different applications is the non-negativity of the principal components~(\cite{montanari2016non}). In this scenario, one needs to solve the PCA or the RPCA under the additional constraints $\bu, \bv\geq 0$. While the non-negative PCA has been recently studied in~\cite{montanari2016non}, the main focus of our work is on its robust variant, where the noise matrix is assumed to be sparse and the goal is the \textit{exact} recovery of the non-negative vectors $\bu$ and $\bv$. 
	Note that the non-negativity of principal components naturally arises in many real-world problems. In what follows, we will present two classes of real-world applications for which the non-negative RPCA is useful.
	
	\vspace{2mm}
	\noindent{\bf 1. Non-negative matrix factorization:} Extracting the dominant principal component of a symmetric or asymmetric data matrix appears in many applications and the examples are ubiquitous. For instance, an important problem in astronomy is the recovery of non-negative astronomical signals from the covariance matrix of photometric observations~(\cite{ren2018non}). The measured data samples are prone to sparse and random outliers. Similarly, one can extract moving objects from video frames via non-negative matrix factorization by treating the background as the dominant low-rank component in the video frames and the moving object as sparse noise (the non-negativity of the data is due to the non-negative values of the pixels)~(\cite{lee1999learning, candes2011robust}). We will conduct a case study on this application later in the paper.
	
	\vspace{2mm}
	\noindent{\bf 2. Gene networks:} Gene activities can be captured by the samples collected from different organs, and are described by multi-spiked models~(\cite{lazzeroni2002plaid}):
	\begin{equation}
	X = X_0+\sum_{i = 1}^k \bu_{(i)}\bv_{(i)}^\top
	\end{equation}
	where $(i,j)^{\text{th}}$ entry of $X$ measures the strength of the participation of gene $i$ in sample $j$ and $X_0$ is an offset. Furthermore, $k$ is the number of the gene-block, and $\bu_{(i)}$ and $\bv_{(i)}$ measure the participation of different genes and samples in the $i^{\text{th}}$ gene-block. The participation vectors are non-negative and the measurements can be subject to malfunctioning of the measurement tools. Therefore, the problem of obtaining $\bu_{(i)}$ and $\bv_{(i)}$ can be cast as a non-negative RPCA with multiple principal components.
	
	\vspace{2mm}
	The seminal work by~\cite{candes2011robust} proposes a sparsity promoting convex relaxation for the RPCA that is capable of the exact recovery of $\bu$ and $\bv$. Upon defining $W = \bu\bv^\top$, the convex relaxation of the RPCA is defined as
	\begin{align}\label{opt_rpca}
	\min_{W\in\mathbb{R}^{m\times n}}\ \ & \|W\|_*+\lambda\|\mathcal{P}_{\Omega}(X-W)\|_1
	\end{align}
	where $\|W\|_*$ is the nuclear norm of $W$, serving as a penalty on the rank of the recovered matrix $W$, and $\|\cdot\|_1$ is used to denote the element-wise $\ell_1$ norm. Furthermore, $\mathcal{P}_{\Omega}(\cdot)$ is the projection onto the set of matrices with the same support as the measurement set $\Omega$. Therefore, upon defining $S = X-W$ as the corruption or noise matrix, $\|\mathcal{P}_{\Omega}(X-W)\|_1$ plays the role of promoting sparsity in the estimated noise matrix. After finding an optimal value of $W$, the matrix can then be decomposed into the desired vectors $\bu$ and $\bv$, provided that the relaxation is exact. Notice that the problem is convexified via lifting from $n+m$ variables on $(\bu,\bv)$ to $nm$ variables on $W$. Despite the convexity of the lifted problem, its dimension makes it prohibitive to solve in high-dimensional settings. To circumvent this issue, one popular approach is to resort to an alternative formulation, inspired by~\cite{burer2003nonlinear} (commonly known as the Burer-Monteiro technique):
	\begin{align}\label{norm1}
	\min_{\bu\in\mathbb{R}_+^{m}, \bv\in\mathbb{R}_+^{n}}\quad \|\mathcal{P}_{\Omega}(X-\bu\bv^\top)\|_1
	\end{align}
	Despite the non-convexity of~\eqref{norm1}, its smooth counterpart (with or without non-negativity constraints) defined as
	\begin{align}\label{norm2}
	\min_{\bu\in\mathbb{R}^{m}, \bv\in\mathbb{R}^{n}}\quad \underbrace{\|\mathcal{P}_{\Omega}(X-\bu\bv^\top)\|^2_F}_{g(\bu,\bv)}
	\end{align}
	has been widely used in matrix completion/sensing and is known to possess \textit{benign global landscape}, i.e., every local solution is also global and every saddle point has a direction with a strictly negative curvature~(\cite{bhojanapalli2016global, ge2016matrix, ge2017no}). This will be stated below.
	\begin{theorem}[Informal, Benign Landscape~(\cite{ge2017no})]\label{thm_diff}
		Under some technical conditions, a regularized version of~\eqref{norm2} has benign landscape: every local minimum is global and every saddle point has a direction with a strictly negative curvature.
	\end{theorem}
	
	In particular, both symmetric and asymmetric matrix completion (or matrix sensing) under dense Gaussian noise can be cast as~\eqref{norm2} and in light of the above theorem, they have benign landscape. However, it is well-known that such smooth norms are incapable of correctly identifying and rejecting sparse-but-large noise/outliers in the measurements.
	
	Despite the generality of Theorem~\ref{thm_diff} within the realm of smooth norms, it does not address the following important question: 
	\textit{Does the non-smooth and non-negative {rank-1} RPCA~\eqref{norm1} have benign landscape?}

	\subsection{The Issue with the Known Proof Techniques}
	
	
	To understand the inherent difficulty of examining the landscape of~\eqref{norm1}, it is essential to explain why the existing proof techniques for the absence of spurious local minima in matrix sensing/completion cannot naturally be extended to their robust counterparts. In general, the main idea in the literature behind proving the benign landscape of matrix sensing/completion is based on analyzing the gradient and the Hessian of the objective function. More precisely, for every point that satisfies $\nabla g(\bu, \bv) = 0$ and does not correspond to a globally optimal minimum, it suffices to find a \textit{global} direction of descent $\bd$ such that $\mathrm{vec}(\bd)^\top\nabla^2g(\bu, \bv)\mathrm{vec}(\bd)<0$, where $\mathrm{vec}(\bd)$ is the vectorized version of $\bd$ and $\nabla^2g(\bu, \bv)$ is the Hessian of $g(\bu, \bv)$. Such a direction certifies that every stationary point that is not globally optimal must be either a local maximum or a saddle point with a strictly negative direction. However, this approach cannot be used to prove similar results for~\eqref{norm1} mainly because the objective function of~\eqref{norm1} is non-differentiable and, hence, the Hessian is not well-defined.
	This difficulty calls for a new methodology for analyzing the landscape of the robust and non-smooth PCA; a goal that is at the core of this work.
	
{\section{Contributions} \label{sec:contributions}}

	{In this work, we characterize the landscape of both the symmetric non-negative rank-1 RPCA defined as 
		\begin{equation}\tag{SN-RPCA}\label{snpca}
		\min_{\bu\in\mathbb{R}^n_+}\quad \underbrace{\|\mathcal{P}_{\Omega}(X-\bu\bu^\top)\|_1 + R_{\beta}(\bu)}_{ f_{\mathrm{reg}}(\bu)}
		\end{equation}
		and its asymmetric counterpart defined as 
		\begin{equation}\tag{AN-RPCA}\label{anpca}
		\min_{\bu\in\mathbb{R}^m_+, \bv\in\mathbb{R}^n_+}\quad \underbrace{\|\mathcal{P}_{\Omega}(X-\bu\bv^\top)\|_1+R_{\beta}(\bu,\bv)}_{f_{\mathrm{reg}}(\bu, \bv)}
		\end{equation}
		In particular, we fully characterize the stationary points of these optimization problems, under both deterministic and probabilistic models for the measurement index $\Omega$ and the noise matrix $S$. The functions $R(\bu)$ and $R(\bu,\bv)$ are regularization functions that prevent the solutions from \textit{blowing up}; roughly speaking, they penalize the points whose norm is greater than $\beta$, but do not change the landscape otherwise. The exact definitions of these regularization functions will be presented later in Section~\ref{sec6}. 
		
		\begin{remark}\label{remark_rankr}
			The focus of this paper is on the symmetric and non-symmetric RPCA under the rank-1 spiked model. A natural extension to this model is its rank-$r$ variant:
			\begin{align}
			X = UV^\top+S
			\end{align}
			where $U := \begin{bmatrix}
			\bu_1 & \cdots & \bu_r
			\end{bmatrix}\in\mathbb{R}^{m\times r}_+$ and $V := \begin{bmatrix}
			\bv_1 & \cdots & \bv_r
			\end{bmatrix}\in\mathbb{R}^{n\times r}_+$ are non-negative matrices encompassing the $r$ principal components of the model (the symmetric version can be defined in a similar manner). Furthermore, similar to the rank-1 case, $S$ is a sparse noise matrix. Under this rank-$r$ spiked model, the aim of the non-negative \textbf{rank-$\bf r$} RPCA is to recover the non-negative matrices $U$ and $V$ given a subset of the elements of the noisy measurement matrix $X$. In Section~\ref{sec:rankr}, we will elaborate on the technical difficulties behind this extension. In addition, we will provide some empirical evidence to support that the developed results may hold for the general non-negative rank-$r$ RPCA with $r\geq 2$. 
		\end{remark}
	}
	
	\begin{definition}
		Given the set $\Omega$, two graphs are defined below:
		\begin{itemize}
			\item[-] The sparsity graph $\mathcal{G}(\Omega)$ induced by $\Omega$ for an instance of~\eqref{snpca} is defined as a graph with the vertex set $V :=\{1,2,...,n\}$ that includes an edge $(i,j)$ if $(i,j)\in\Omega$.
			\item[-] The bipartite sparsity graph $\mathcal{G}_{m, n}(\Omega)$ induced by $\Omega$ for an instance of~\eqref{anpca} is defined as a graph with the vertex partitions $V_u :=\{1,2,...,m\}$ and {$V_v :=\{m+1,m+2,...,m+n\}$} that includes an edge $(i,j)$ if $(i,j-m)\in\Omega$.
		\end{itemize}
		Furthermore, define $\Delta(\mathcal{G}(\Omega))$ and $\delta(\mathcal{G}(\Omega))$ as the maximum and minimum degrees of the nodes in $\mathcal{G}(\Omega)$, respectively. Similarly, $\Delta(\mathcal{G}_{m,n}(\Omega))$ and $\delta(\mathcal{G}_{m,n}(\Omega))$ are used to refer to the maximum and minimum degrees of the nodes in $\mathcal{G}_{m,n}(\Omega)$, respectively.
	\end{definition}
	\begin{definition}
		The sets of~\textbf{bad/corrupted} and~\textbf{good/correct} measurements are defined as ${B} = \{(i,j)| (i,j)\in\Omega, S_{ij}\not=0\}$ and ${G} = \{(i,j)| (i,j)\in\Omega, S_{ij}=0\}$, respectively.
	\end{definition}
	
	Based on the above definitions, the sparsity graph is allowed to include self-loops.
	For a positive vector $\bx$, we denote its maximum and minimum values with $x_{\max}$ and $x_{\min}$, respectively. Furthermore, define $\kappa(\bx) = \frac{x_{\max}}{x_{\min}}$ as the condition number of the vector $\bx$.
	The first result of this paper develops deterministic conditions on the measurement set $\Omega$ and the sparsity pattern of the noise matrix $S$ to guarantee that the positive {rank-1} RPCA has benign landscape. Let $\bu ^*$ and $(\bu ^*, \bv ^*)$ denote the true principal components of~\eqref{snpca} and~\eqref{anpca}, respectively.
	
	\begin{theorem}[Informal, Deterministic Guarantee]\label{thm_inf_det}
		{Assuming that $\bu ^*, \bv ^*>0$, there exist regularization functions $R(\bu)$ and $R(\bu,\bv)$ such that the following statements hold with overwhelming probability:
			\begin{itemize}
				\item[1.]~\eqref{snpca} has no spurious local minimum and has a unique global minimum that coincides with the true component, provided that $\mathcal{G}(G)$ has \textbf{no bipartite component} and 
				\begin{equation}
				\kappa(\bu^*)^4\Delta(\mG(B))\lesssim\delta(\mG(G))
				\end{equation}
				\item[2.]~\eqref{anpca} has no spurious local minimum and has a unique global minimum that coincides with the true components, provided that $\mathcal{G}_{m,n}(G)$ is \textbf{connected} and
				\begin{equation}
				\max\left\{\kappa(\bu^*)^4, \kappa(\bv^*)^4\right\}\Delta(\mG_{m,n}(B))\lesssim\delta(\mG_{m,n}(G))
				\end{equation}
		\end{itemize}}
	\end{theorem}
	Theorem~\ref{thm_inf_det} puts forward a set of deterministic conditions for the absence of spurious local solutions in~\eqref{snpca} and~\eqref{anpca} as well as the uniqueness of the global solution. Notice that no upper bound is assumed on the values of the nonzero entries in the noise matrix. The reasoning behind the conditions imposed on the minimum and maximum degrees of the nodes in the sparsity graph of the measurement set is to ensure the identifiability of the problem. We will elaborate more on this subtle point later in Section~\ref{sec6}. Furthermore, we will show later in the paper that some of the conditions delineated in Theorem~\ref{thm_inf_det}---such as the strict positivity of $\bu ^*$ and $\bv ^*$, as well as the absence of bipartite components in $\mG(G)$ for~\eqref{snpca}---are also necessary for the exact recovery.
	
	The second main result of this paper investigates~\eqref{snpca} and~\eqref{anpca} under random sampling and noise structures. In particular, suppose that each element (in the symmetric case, each element of the upper triangular part) of $S$ is nonzero with probability $d$. Then, for every $(i,j)$, we have
	\begin{equation}
	X_{ij} = \left\{
	\begin{array}{ll}
	u^*_iv^*_j& \text{with probability}\ 1-d\\
	\text{arbitrary}& \text{with probability}\ d
	\end{array} 
	\right.
	\end{equation}
	Furthermore, suppose that every element of $X$ is measured with probability $p$. In other words, every $(i,j)$ belongs to $\Omega$ with probability $p$. Finally, we assume that the noise and sampling events are independent.
	
	\begin{theorem}[Informal, Probabilistic Guarantee]\label{thm_inf_prob}
		{Assuming that $\bu ^*, \bv ^*>0$, there exist regularization functions $R(\bu)$ and $R(\bu,\bv)$ such that the following statements hold with overwhelming probability:
			\begin{itemize}
				\item[1.]~\eqref{snpca} has no spurious local minimum and has a unique global minimum that coincides with the true component, provided that
				\begin{equation}\label{upper}
				p\gtrsim\frac{\kappa(\bu^*)^4\log n}{n},\qquad d\lesssim\frac{1}{\kappa(\bu^*)^4}
				\end{equation}
				\item[2.]~\eqref{anpca} has no spurious local minimum and has a unique global minimum that coincides with the true components, provided that
				\begin{equation}
				p\gtrsim\frac{\kappa(\bw^*)^4n\log n}{m^2},\qquad d\lesssim\frac{r}{\kappa(\bw^*)^4}
				\end{equation}
				where $\bw^* = \begin{bmatrix}
				{\bu^*}^\top & {\bv^*}^\top
				\end{bmatrix}^\top$, $r = m/n$, and $n\geq m$.
		\end{itemize}}
	\end{theorem}
	
	A number of interesting corollaries can be obtained based on Theorem~\ref{thm_inf_prob}. For instance, it can be inferred that the exact recovery is guaranteed even if the number of grossly corrupted measurements is on the same order as the total number of measurements, provided that $\frac{u^*_{\max}}{u^*_{\min}}$ is uniformly bounded from above.
	
	In addition to the absence of spurious local minima and the uniqueness of the global minimum, the next proposition states that the true solution can be recovered via local search algorithms for non-smooth optimization.
	\begin{proposition}[Informal, Global Convergence]\label{prop_global}
		Under the assumptions of Theorem~\ref{thm_inf_det} and~\ref{thm_inf_prob}, local search algorithms converge to the true solutions of~\eqref{snpca} and~\eqref{anpca} with overwhelming probability.
	\end{proposition}
	
	
	Starting from Section~\ref{sec2}, we will delve into the detailed analysis of the symmetric and asymmetric non-negative RPCA. In particular, we will analyze~\eqref{snpca} and~\eqref{anpca} under different deterministic and probabilistic settings and provide formal versions of Theorems~\ref{thm_inf_det} and~\ref{thm_inf_prob}.
	
	\section{Numerical Results}\label{sec:num}
{In this section, we demonstrate the efficacy of the above-mentioned results in different experiments. To this goal, first we briefly introduce the recently developed sub-gradient method~\cite{li2018nonconvex} that is specifically tailored to non-smooth and non-convex problems, such as those considered in this paper. The main advantage of the sub-gradient algorithm compared to other state-of-the-art methods is its extremely simple implementation; we present a sketch of the algorithm for solving the non-symmetric positive RPCA below\footnote{Note that this is a slightly modified version of the sub-gradient algorithm in~\cite{li2018nonconvex} to ensure the positivity of the iterates.} (the symmetric version can be solved using a similar algorithm with slight modifications):
	
	\vspace{2mm}
	{\begin{algorithm}[H]
			\SetAlgoLined
			{\textbf{Initialization:} Strictly positive initial point $\bw_0^\top = \begin{bmatrix}
				\bu_0^\top & \bv_0^\top
				\end{bmatrix}^\top$ and step size $\mu_0$}\;
			\For{$k = 0,1,\dots$}{
				set $\bd_k$ as a sub-gradient of $f_{\mathrm{reg}}(\bu_0,\bv_0)$ defined in~\eqref{anpca}\;
				set $\mu_k$ according to a geometrically diminishing rule such that $\bw_k-\mu_k\bd_k$ is strictly positive\;
				set $\bw_{k+1} = \bw_k-\mu_k\bd_k$\;
			}
			\caption{Sub-gradient algorithm}\label{algo1}
		\end{algorithm}
	}
	
	\vspace{2mm}
	\noindent It has been shown in~\cite{li2018nonconvex} that, under certain conditions on the initial point $\bw_0$, the initial step size $\mu_0$, and the update rule for $\mu_k$, the iterates $\bw_0,\bw_1,\dots$ converge to the globally optimal solution at \textit{linear} rate, provided that $\bw_0$ is sufficiently close to the optimal solution. The closeness of $\bw_0$ to $\bw^*$ is required partly to avoid becoming stuck at a spurious local minima. This requirement can be relaxed for the positive RPCA due to the absence of undesired spurious local solutions, as proven in this paper. It is also worthwhile to mention that, even though we use the sub-gradient algorithm to solve the positive RPCA, it will be shown in Section~\ref{sec8} that the results of this paper guarantee that a large class of local-search algorithms converge to the globally optimal solution of~\eqref{snpca} or~\eqref{anpca}. 
	
	All of the following simulations are run on a laptop computer with an Intel Core i7 quad-core 2.50 GHz CPU and 16GB RAM. The reported results are for a serial implementation in MATLAB R2017b.
}

\begin{figure*}
	\centering
	\subfloat[RPCA]{\label{fig_rpca}
		\includegraphics[width=.46\columnwidth]{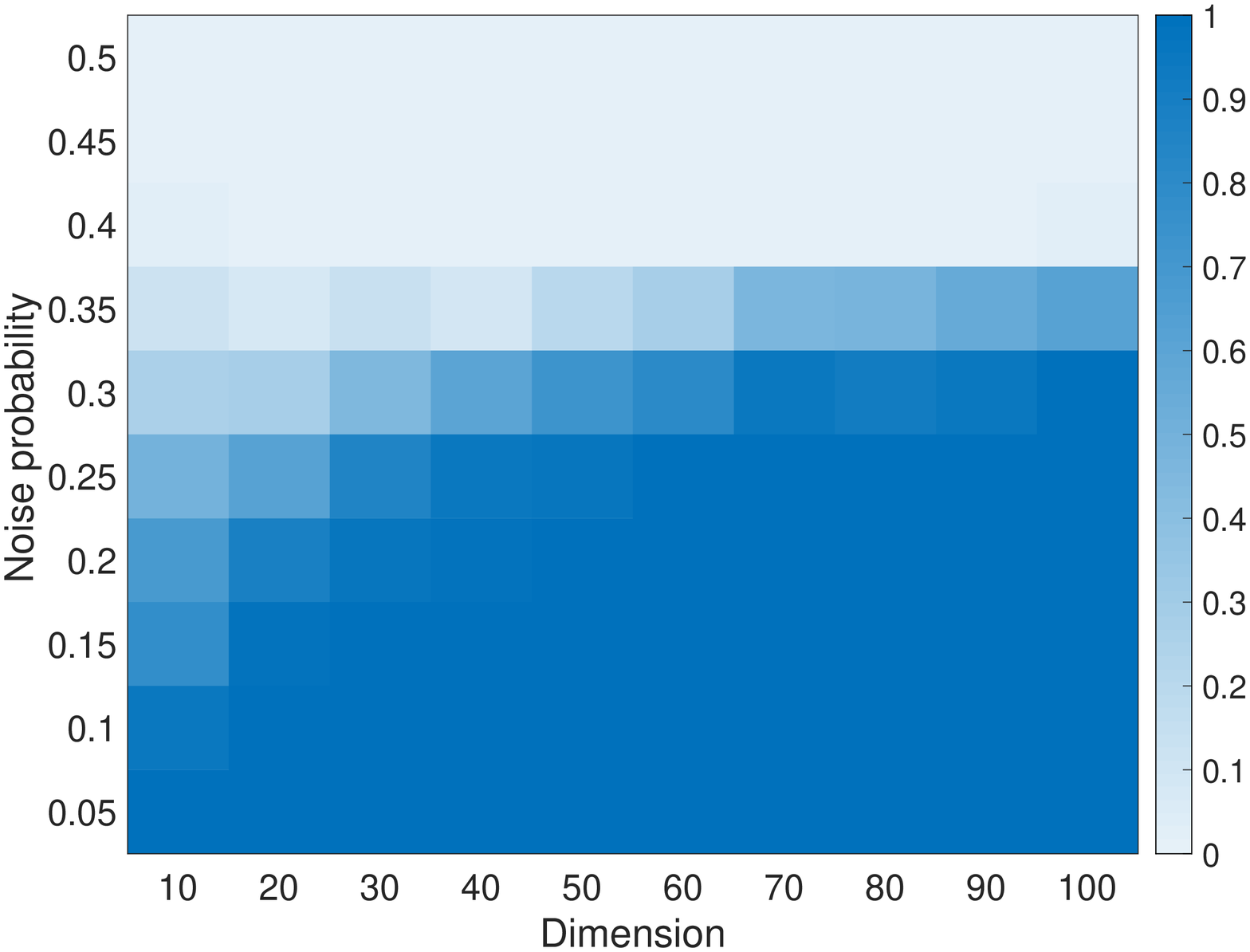}}
	\hspace{0cm}
	\subfloat[PCA]{\label{fig_runtime}
		\includegraphics[width=.5\columnwidth]{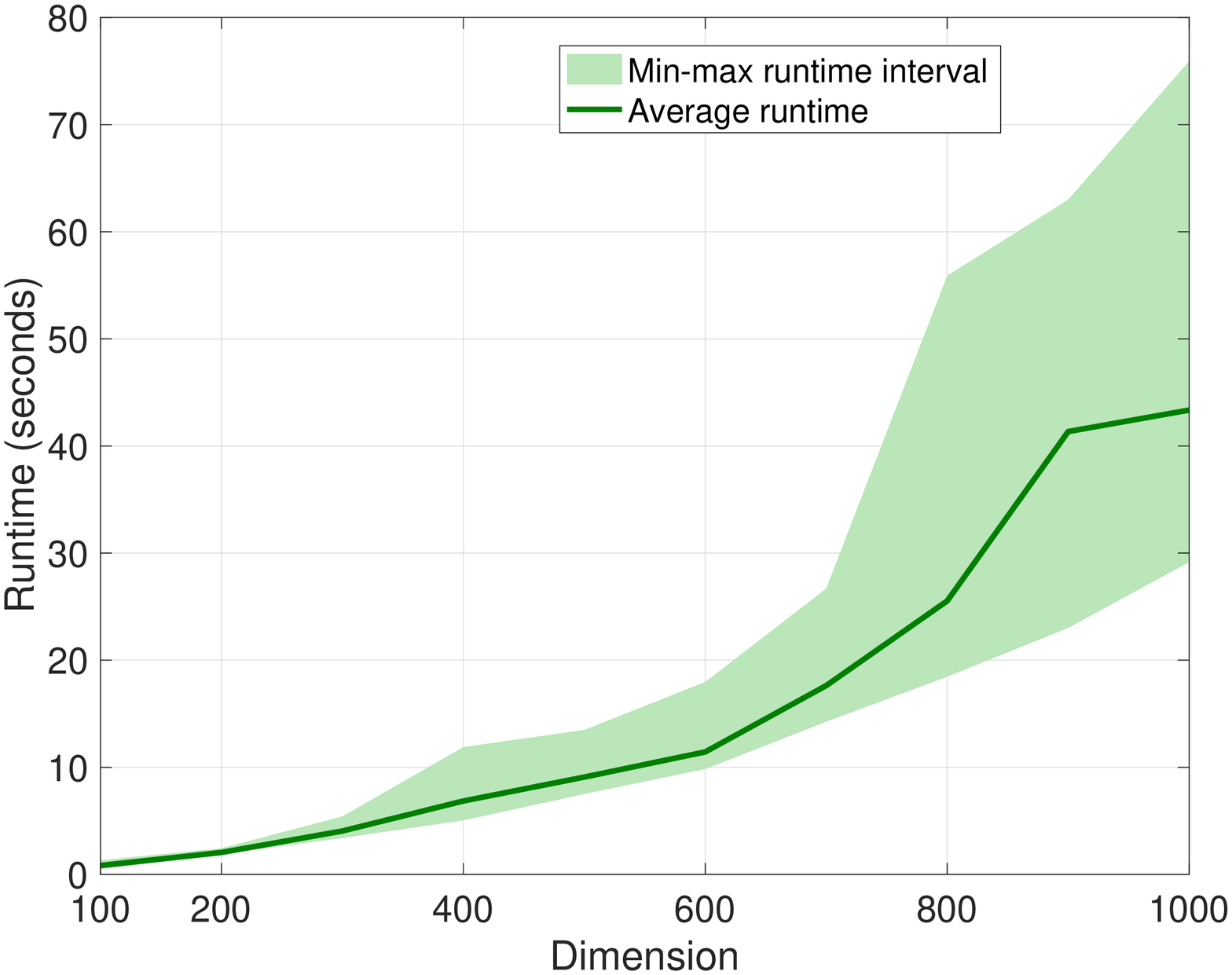}}
	
	\caption{ \footnotesize (a) The performance of the randomly initialized sub-gradient method for~\eqref{snpca}. The intensity of the color is proportional to the exact recovery rate of the true solution (darker blue implies higher recovery rate). (b) The runtime of the sub-gradient method for~\eqref{snpca}. For each dimension, it shows the average runtime and its min-max interval over 100 independent trials.}
	\label{rpca_vs_pca}
\end{figure*}
%
\subsection{Exact Recovery:}\label{subsec:exact} To demonstrate the strength of the above-mentioned results, we consider thousands of randomly generated instances of the positive {rank-1} RPCA with different sizes and noise levels. In particular, the dimension of the instances ranges from $10$ to $100$. For each instance, the elements of $\bu^*$ are uniformly chosen from the interval $[0,2]$. Note that $\bu^*$ will be strictly positive with probability one. Furthermore, each element of the upper triangular part of the symmetric noise matrix $S$ is set to $2$ with probability $d$ and $0$ with probability $1-d$. Figure~\ref{fig_rpca} shows the performance of randomly initialized sub-gradient method for the symmetric positive {rank-1} RPCA. {We declare that a solution is recovered exactly if $\|\bu\bu^\top - \bu^*{\bu^*}^\top\|_F/\|\bu^*{\bu^*}^\top\|_F\leq 10^{-4}$}. 
For each dimension and noise probability, we consider 100 randomly generated instances of the problem and demonstrate its exact recovery rate. The heatmap shows the exact recovery rate of the sub-gradient method, when directly applied to~\eqref{snpca}. It can be observed that the algorithm has recovered the globally optimal solution even when $35\%$ of the entries in the data matrix were severely corrupted with the noise. In contrast, even a highly sparse additive noise in the data matrix prevents the sub-gradient method from recovering the true solution, when applied to the smooth problem~\eqref{norm2}. {Figure~\ref{fig_runtime} shows the graceful scalability of the sub-gradient algorithm when applied to~\eqref{snpca}. It can be seen that the algorithm is highly efficient. In particular, its average runtime varies from $0.88$ seconds for $n = 100$ to $43.20$ seconds for $n = 1000$.}

\vspace{5mm}
\subsection{The Emergence of Local Solutions} 
Recall that $\bu^*$ and $\bv^*$ are both assumed to be strictly positive. In what follows, we will illustrate that relaxing these conditions to non-negativity gives rise to spurious local solutions. Consider an instance of the symmetric {non-negative rank-1} RPCA with the parameters
\begin{equation}
\bu^* = \begin{bmatrix}
1 & 1 & 0
\end{bmatrix}^\top,\qquad S = 0,\qquad \Omega = \{1,2,3\}^2\backslash\{(3,3)\}
\end{equation}
Notice that $\bu^*$ consists of two strictly positive and one zero entries. Furthermore, this is a noiseless scenario where $\Omega$ consists of all possible measurements except for one. To examine the existence of spurious local solutions in this example, $10000$ randomly initialized trials of the sub-gradient method is ran and the normalized distances between the obtained and true solutions are displayed in Figure~\ref{hist}. Based on this histogram, about $20\%$ of the trials converge to spurious local solutions, implying that they are ubiquitous in this instance. This experiment shows why the positivity of the true solution is crucial and cannot be relaxed. We will formalize and prove this statement later in Section~\ref{sec3}.  

\begin{figure}
	\centering
	\includegraphics[width=.45\columnwidth]{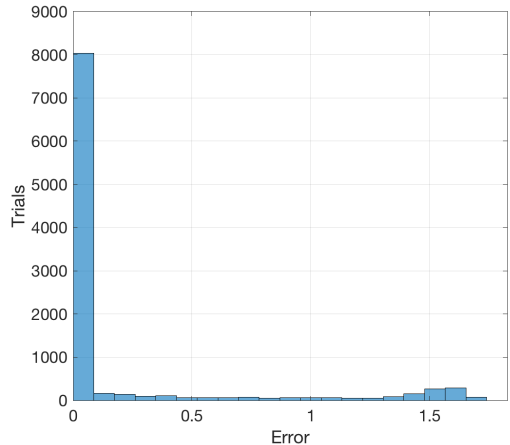}
	\caption{ \footnotesize The normalized distance between the obtained solution using randomly initialized sub-gradient method and the true solution.}
	\label{hist}
\end{figure}

\vspace{5mm}
\subsection{Moving Object Detection} In video processing, one of the most important problems is to detect anomaly or moving objects in different frames of a video. In particular, given a video sequence, the goal is to separate the nearly-static or slowly-changing background from the dynamic foreground objects~(\cite{cucchiara2003detecting}).  Based on this observation,~\cite{candes2011robust} has proposed to model the background as a low-rank component, and the dynamic foreground as the sparse noise. In particular, suppose that the video sequence consists of $d_f$ gray-scale frames, each with the resolution of $d_m\times d_n$ pixels. The data matrix $X$ is defined as an asymmetric $d_md_n \times d_f$ matrix whose $i^{\text{th}}$ column is the vectorized version of the $i^{\text{th}}$ frame. Therefore, the moving object detection problem can be cast as the recovery of the non-negative vectors $\bu\in\mathbb{R}^{d_md_n}_+$ and $\bv\in\mathbb{R}^{d_f}_+$, as well as the sparse matrix $S\in\mathbb{R}^{d_md_n\times d_f}$, such that
\begin{equation}\label{model}
X \approx \bu\bv^\top+S
\end{equation}
Note that the background may not always have a rank-1 representation. However, we will show that~\eqref{model} is sufficiently accurate if the background is relatively static. Furthermore, notice that when the background is completely static, the elements of $\bv$ should be equal to one. However, this is not desirable in practice since the background may change due to varying illuminations, which can be captured by the variable vector $\bv$. Each entry of $X$ is an integer between 0 (darkest) and 255 (brightest). To ensure the positivity of the true components, we increase each element of $X$ by 1 without affecting the performance of the method.

The considered test case is borrowed from the work by~\cite{toyama1999wallflower}\footnote{The video frames are publicly available at~\url{https://www.microsoft.com/en-us/research/project/test-images-for-wallflower-paper/}.} and is a sequence of video frames taken from a room, where a person walks in, sits on a chair, and uses a phone. We consider 100 gray-scale frames of the sequence, each with the resolution of $120\times 160$ pixels. Therefore, $X$, $\bu$, and $\bv$ belong to $\mathbb{R}^{19,200\times 100}_+$, $\mathbb{R}^{19,200}_+$, and $\mathbb{R}^{100}_+$, respectively. Figure~\ref{movedObj} shows that the sub-gradient method with a random initialization can recover the moving object, which is in accordance with the theoretical results of this paper. 

\begin{figure}
	\centering
	\subfloat{\label{bw1}
		\includegraphics[width=.3\columnwidth]{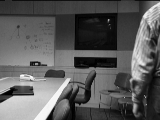}}
	\subfloat{\label{bw2}
		\includegraphics[width=.3\columnwidth]{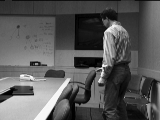}}
	\subfloat{\label{bw3}
		\includegraphics[width=.3\columnwidth]{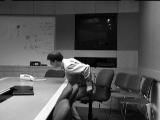}}
	
	\subfloat{\label{pic1_2}
		\includegraphics[width=.3\columnwidth]{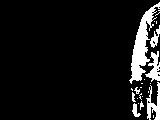}}
	\subfloat{\label{pic2_2}
		\includegraphics[width=.3\columnwidth]{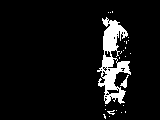}}
	\subfloat{\label{pic3_2}
		\includegraphics[width=.3\columnwidth]{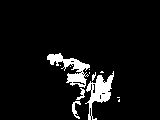}}
	
	\caption{ \footnotesize The performance of the sub-gradient method in the moving object detection problem. The first row shows 3 out of 100 gray-scale frames in the studied test case that contain the moving objects. The second row shows the outcome of~\eqref{snpca} solved using randomly initialized sub-gradient method. }\label{movedObj}
\end{figure}
	
	{\section{Related Work}}

	\vspace{2mm}
	{\subsection{Non-convex and Low-rank Optimization}}
	\vspace{2mm}
	A considerable amount of work has been carried out to understand the inherent difficulty of solving low-rank optimization problems both locally and globally.
	
	\vspace{2mm}
	\noindent{\bf Convexification:}
	Recently, there has been a pressing need to develop efficient methods for solving large-scale nonconvex optimization problems that naturally arise in data analytics and machine learning~(\cite{dumais1998inductive, sharif2014cnn, bottou2018optimization, zhang2018large, olfat2017spectral}). One promising approach for making these large-scale problems more {tractable} is to resort to their convex surrogates; these methods started to receive a great deal of attention after the seminal works by~\cite{donoho2006most} and~\cite{candes2006stable} on the \textit{compressive sensing} and have been extended to emerging problems in machine learning, such as fairness~(\cite{olfat2017spectral}), robust polynomial regression~(\cite{molybog2018conic, Madani2018conic}), and neural networks~(\cite{bach2017breaking}), to name a few. Nonetheless, the size of today's problems has been a major impediment to the tractability of these methods. In practice, the dimension of the real-world problems is overwhelmingly large, often surpassing the ability of these seemingly efficient convex methods to solve the problem in a reasonable amount of time. Due to this so-called \textit{curse of dimensionality}, the common practice is to deploy fast local search algorithms directly applied to the original nonconvex problem with the hope of converging to acceptable solutions. Roughly speaking, these methods can only guarantee the local optimality, thus exposing themselves to potentially large optimality gaps. However, a recent line of work has shown that a surprisingly large class of nonconvex problems, including matrix completion/sensing~(\cite{bhojanapalli2016global, ge2016matrix, ge2017no, zhu2017global}), phase retrieval~(\cite{sun2018geometric}), and dictionary recovery~(\cite{sun2017complete}) have \textit{benign global landscape}, i.e., every local solution is also global and every saddle point has a direction with a strictly negative curvature (see~\cite{chi2018nonconvex} for a comprehensive survey on the related problems). {More recently, the work by~\cite{zhang2018primal} has introduced a unified framework that shows the benign landscape of nonconvex low-rank optimization problems with general loss functions, provided that they satisfy certain restricted convexity and smoothness properties.} This enables most of the saddle-escaping local search algorithms to converge to a global solution, thereby resulting in a zero optimality gap~(\cite{ge2015escaping}).
	
	
	\vspace{2mm}
	\noindent{\bf Benign landscape:} As mentioned before, it has been recently shown that many low-rank optimization problems can be cast as smooth-but-nonconvex optimization problems that are free of spurious local minima. These methods heavily rely on the notion of \textit{restricted isometry property} (RIP)---a property that was initially introduced by~\cite{candes2005decoding} and has been used ever since as a metric to measure a norm-preserving property of the objective function. In general, these methods have two major drawbacks: 1) they can only target a narrow set of nearly-isotropic instances~(\cite{zhang2018much}), and 2) their proof technique depends on the differentiability of the objective function; a condition that is not satisfied for non-smooth norms, such as $\ell_1$. To the best of our knowledge, the work by~\cite{josz2018theory} is the only one that studies the landscape of the $\ell_1$ minimization problem, where the authors consider the tensor decomposition problem under the full and perfect measurements. Our work is somewhat related to ~\cite{ma2018gradient} that derives similar conditions for the absence of spurious local solution of the non-negative rank-1 matrix completion but for the smooth Frobenius norm minimization problem.
	
	\vspace{2mm}
	\noindent{\bf PCA with prior information:} With an exponential growth in the size and dimensionality of the real-world datasets, it is often required to exploit the additional prior information in the PCA. In many real-world applications, prior knowledge from the underlying physics of the problem---such as non-negativity~(\cite{montanari2016non}), sparsity~(\cite{zou2006sparse}), robustness~(\cite{candes2011robust}), and nonlinearity~(\cite{gorban2008principal})---can be taken into account to perform more efficient, consistent, and accurate PCA. 
	
	\vspace{2mm}
	
	\noindent{\bf Numerical algorithms for non-smooth optimization:} Numerical algorithms for non-smooth optimization problems can be dated back to the work by Clarke on the extended definitions of gradients and directional derivatives, commonly known as generalized derivatives~(\cite{clarke1990optimization}). Intuitively, for non-smooth functions, the gradient in the classical sense seize to exist at a subset of the points in the domain. The Clarke generalized derivative is introduced to circumvent this issue by associating a convex differential to these points, even if the original problem is non-convex. In the domain of unconstrained non-smooth optimization, earlier works have introduced simple algorithms that converge to approximate Clarke-stationary points~(\cite{goldstein1977optimization, chaney1978extension}). More recent methods take advantage of the fact that many non-smooth optimization problems are smooth in every open dense subset of their domains. This implies that the objective function is smooth with probability one at a randomly drawn point. This observation lays the groundwork for several gradient-sampling-based algorithms for both unconstrained and constrained non-smooth optimization problems~(\cite{burke2005robust, curtis2012sequential}). {As mentioned before, a sub-gradient method has been recently proposed by~\cite{li2018nonconvex} for solving the RPCA, where the authors prove linear convergence of the algorithm to the true components, provided that the initial point is chosen sufficiently close to the globally optimal solution.}
	
	\vspace{2mm}
	{\subsection{Comparison to the Existing Results on RPCA}
		Similar to the non-convex matrix sensing and completion, most of the existing results on the RPCA work on a \textit{lifted} space of the variables via different convex relaxations and they do not incorporate the positivity constraints in the problem. In what follows, we will explain the advantages of our proposed method compared to these results.
		
		\vspace{2mm}
		
		\noindent{\bf Positivity constraints:} In the present work, we show that the positivity of the true components is both sufficient and (almost) necessary for the absence of spurious local solutions. We use this prior knowledge to obtain sharp deterministic and probabilistic guarantees on the absence of spurious local minima for the RPCA based on the Burer-Monteiro formulation. For instance, we show that up to a constant factor of the measurements can be grossly corrupted and yet they do not introduce any spurious local solution. Considering the fact that these results heavily rely on the positivity of the true components, it is unclear if similar ``no spurious local minima'' results hold for the general case without the positivity assumption. The statistical properties of these types of constraints have also been shown to be useful in the classical PCA by~\cite{montanari2016non}, where the authors show that by imposing positivity constraints on the principal components, one can guarantee its consistent recovery with smaller signal-to-noise ratio. 
		It is also worthwhile to mention that the incorporation of the non-negativity/positivity constraints in the low-rank matrix recovery can be traced back to some earlier works on the non-negative matrix factorization problem~(\cite{lee1999learning, hoyer2004non}).
		
		\vspace{2mm}
		\noindent{\bf Computational savings:} Similar to the convexification techniques in nonconvex optimization, most of the classical results on the RPCA \text{relax} the inherent non-convexity of the problem by lifting it to higher dimensions~(\cite{candes2011robust, chandrasekaran2011rank, zhou2010stable, hsu2011robust}). In particular, by moving from vector to matrix variables, they guarantee the convexity of the problem at the expense of significantly increasing the number of variables. In this work, we show that such lifting is not necessary for the positive rank-1 RPCA since---despite the non-convexity of the problem---it is free of spurious local solutions and, hence, simple local search algorithms converge to the true components when directly applied to its original formulation.
		
		\vspace{2mm}
		\noindent
		{\bf Sharp guarantees with mild conditions:} In general, most of the existing results on RPCA for guaranteeing the recovery of the true components fall into two categories. First, a large class of methods rely on some deterministic conditions on the spectra of the dominant components and/or the structure of the sparse noise~(\cite{hsu2011robust, chandrasekaran2011rank, yi2016fast}). For instance, the works by~\cite{hsu2011robust, chandrasekaran2011rank} require the regularization coefficient to be within a specific interval that is defined in terms of the true principal components. Furthermore, the algorithm proposed by~\cite{yi2016fast} requires prior knowledge on the density of the sparse noise matrix. Although being theoretically significant, these types of conditions cannot be easily verified and met in practice. With the goal of bypassing such stringent conditions, the second category of research has studied the RPCA under probabilistic models. These types of guarantees were popularized by~\cite{candes2011robust,wright2009robust} and they do not rely on any prior knowledge on the true components or the density of the noise matrix. However, their success is contingent upon specific random models on the sparse noise or the spectra of the true components, neither of which may be satisfied in practice.
		
		In contrast, the method proposed here does not rely on any prior knowledge on the true solution, other than the availability of an upper bound on the maximum absolute value of the elements in the principal components\footnote{Note that in most cases, these types of upper bounds can be immediately inferred by the domain knowledge; see e.g. our discussion on the moving object detection problem.}. Furthermore, unlike the previous works, our results encompass \textit{both deterministic and probabilistic} models under random sampling.
		
	}
	
	{\section{Preliminaries}}\label{sec2}
	
	A \textbf{directional derivative} of a locally Lipschitz and possibly non-smooth function $h(\bx)$ at $\bx$ in the direction $\bd$ is defined as 
	\begin{equation}
	h'(\bx,\mathbf{d}) := \underset{\begin{subarray}{c}
		t\downarrow 0
		\end{subarray}}{\lim}\frac{h(\bx+t\mathbf{d})-h(\bx)}{t} 
	\end{equation}
	upon existence. Based on this definition, $\bar\bu$ is \textbf{directional-minimum-stationary} (or D-min-stationary) for~\eqref{snpca} if $f'(\bar\bu,\mathbf{d})\geq 0$ for every \textit{feasible} direction $\bd$, i.e., a direction that satisfies $d_i\geq 0$ when $u_i = 0$ for every index $i$. Similarly, $\bar\bu$ is \textbf{directional-maximum-stationary} (or D-max-stationary) for~\eqref{snpca} if $f'(\bar\bu,\mathbf{d})\leq 0$ for every feasible $\bd$. Finally, $\bar\bu$ is \textbf{directional-stationary} (or D-stationary) for~\eqref{snpca} if it is either D-min- or D-max-stationary\footnote{Note that the notion of D-stationary points is often used in lieu of D-min-stationary in the literature. However, we use a slightly more general definition in this paper to account for the local maxima of~\eqref{snpca}.}.
	
	Every local minimum (maximum) $\bar{\bu}$ should be D-min (max)-stationary for $f(\bu)$. On the other hand, $\bar{\bu}$ cannot be a D-stationary point if $f(\bu)$ has strictly positive and negative directional derivatives at that point. In that case, $\bar{\bu}$ is neither local maximum nor minimum. 
	A solution to a minimization problem is referred to as~\textbf{spurious local} (or simply local) if there exists another feasible point with a strictly smaller objective value; a solution is~\textbf{globally optimal} (or simply global) if no such point exists. 
	
	Finally, a \textbf{vertex partitioning} of a non-empty bipartite graph is the partition of its vertices into two groups such that there exist no adjacent vertices within each group. 
	
	{\bf Notation:} The upper-case, bold lower-case, and lower-case letters are used to show the matrices, vectors, and scalars, respectively. The space of non-negative and real $n\times 1$ vectors and $m\times n$ matrices are denoted by $\mathbb{R}^n_+$ and $\mathbb{R}^{n\times m}_+$, respectively. The symbols $\|{W}\|_1$ and $\|W\|_F$ denote the element-wise $\ell_1$ norm and Frobenius norm of $W$, respectively. The $(i,j)^{\text{th}}$ entry of a matrix $W$ is shown as $W_{ij}$, whereas the $i^{\text{th}}$ entry of a vector ${\bf w}$ is denoted by $w_i$. Given the sequences $f_1(n)$ and $f_2(n)$, the notation $f_1(n) \lesssim f_2(n)$ or equivalently $f_1(n) = O(f_2(n))$ means that there exists a number $c_1\in[0,\infty)$ such that $f_1(n)\leq c_1f_2(n)$ for all $n$. Similarly, the notation $f_1(n) \gtrsim f_2(n)$ or $f_1(n) = \Omega(f_2(n))$ means that there exists a number $c_2>0$ such that $f_1(n)\geq c_2f_2(n)$ for all $n$. 
	The indicator function $\mathbb{I}_{x\geq\alpha}$ takes the value $1$ if $x\geq\alpha$ and $0$ otherwise. For an event $\mathcal{E}$, the notation $\mathbb{P}(\mathcal{E})$ is used to show the probability of its occurrence. { For a random variable $X$, the symbol $\mathbb{E}\{X\}$ shows its expected value. For notational simplicity and unless stated otherwise, we will refer to non-negative (or positive) rank-1 RPCA as non-negative (or positive) RPCA in the sequel.}

	
	\vspace{2mm}
	\section{Base Case: Noiseless Non-negative RPCA}\label{sec3}
	
	In this section, we consider the noiseless version of both symmetric and asymmetric non-negative RPCA. While not entirely obvious, the subsequent arguments are at the core of our proofs for the general noisy case. In the noiseless scenario,~\eqref{snpca} is reduced to
	{\begin{equation}\tag{P1-Sym}\label{p2-sym}
		\min_{\bu\geq 0}\quad \underbrace{\sum_{(i,j)\in\Omega}|u_iu_j-u_i^*u_j^*|}_{f(\bu)}
		\end{equation}}
	%
	For the asymmetric problem~\eqref{anpca}, the solution is invariant to scaling. In other words, if $(\bu, \bv)$ is a solution to~\eqref{anpca}, then $(\frac{1}{q}\bu, q\bv)$ is also a valid solution with the same objective value, for every scalar $q>0$. To circumvent the issue of invariance to scaling, it is common to balance the norms of $\bu$ and $\bv$ by penalizing their difference. Therefore, similar to the works by~\cite{ge2017no, zheng2016convergence, yi2016fast}, we consider the following regularized variant of~\eqref{anpca}: 
	{\begin{equation}\label{p2-asym}
		\min_{\bu\geq 0, \bv\geq 0}\quad \underbrace{\|\mathcal{P}_{\Omega}(X-\bu\bv^\top)\|_1 + \alpha|\bu^\top\bu - \bv^\top\bv|}_{f_{\mathrm{asym}}(\bu, \bv)}
		\end{equation}}
	for an arbitrary constant $\alpha>0$ (note that the positivity of $\alpha$ is the only condition required in this work). To deal with the asymmetric case, we first convert it to a symmetric problem after a simple concatenation of variables.
	Define $\bw = [\bu^\top\ \ \bv^\top]^\top$, $\bw^* = [{\bu ^*}^\top\ \ {\bv^*}^\top]^\top$, and $\bar\Omega = \{(i,j)| (i,j-m)\in\Omega\}$. 
	Based on these definitions, one can symmetrize~\eqref{p2-asym} as follows:
	{\begin{equation}\tag{P1-Asym}\label{p2_asym_sym}
		\min_{\bw\geq 0}\quad  \underbrace{\sum_{(i,j)\in\bar\Omega}|w_iw_j-w^*_iw^*_j|+\alpha\left|\sum_{i=1}^{m}w_i^2-\sum_{j = m+1}^{m+n}w_j^2\right|}_{f_{\mathrm{sym}}(\bw)}
		\end{equation}}
	To simplify the notation, we drop the subscript from $f_{\mathrm{sym}}(\bw)$ whenever there is no ambiguity in the context. 
	\subsection{Deterministic Guarantees}
	{\bf Symmetric case:} First, we introduce deterministic conditions to guarantee a benign landscape for~\eqref{p2-sym}.

	\begin{theorem}\label{thm1}
		Suppose that $\bu^*>0$ and $\mathcal{G}(\Omega)$ has no bipartite component. Then, the following statements hold for \eqref{p2-sym}:
		\begin{itemize}
			\item[1.] It does not have any spurious local minimum;
			\item[2.] The point $\bu = \bu^*$ is the unique global minimum;
			\item[3.] In the positive orthant, the point $\bu = \bu^*$ is the only D-stationary point.
		\end{itemize}
		Additionally, if $\mathcal{G}(\Omega)$ is connected, the following statements hold for \eqref{p2-sym}:
		\begin{itemize}
			\item[4.] The points $\bu = \bu^*$ and $\bu = 0$ are the only D-min-stationary points;
			\item[5.] The point $\bu = 0$ is a local maximum.
		\end{itemize}
	\end{theorem}
	
	The above theorem has a number of important implications for~\eqref{p2-sym}: 1) it has no spurious local solution, 2) $\bu = \bu ^*$ is its unique global solution, and 3) every feasible point $\bu>0$ such that $\bu \not= \bu ^*$ has at least a strictly negative directional derivative. Additionally, if $\mG(\Omega)$ is connected, the feasible points of~\eqref{p2-sym} with zero entries either have a strictly negative directional derivative or correspond to the origin that is a local maximum with a strictly negative curvature. Therefore, these points are not local/global minima and can be easily avoided using local search algorithms.
	
	To prove Theorem~\ref{thm1}, we first need the following important lemma.
	
	\begin{lemma}\label{l1}
		Suppose that $\mG(\Omega)$ has no bipartite component and $\bu ^*>0$. Then, for every D-min-stationary point $\bu$ of~\eqref{p2-sym}, we have $\bu[c] >0$ or $\bu[c] = 0$, where $\bu[c]$ is a sub-vector of $\bu$ induced by the $c^{\text{th}}$ component of $\mG(\Omega)$.
	\end{lemma}
	
	\begin{proof}
		See Appendix~\ref{app_l1}.
	\end{proof}
	
	Now, we are ready to present the proof of Theorem~\ref{thm1}.
	
	\vspace{2mm}
	\noindent{\bf Proof of Theorem~\ref{thm1}:}
	We prove the first three statements. Note that Statement 5 can be easily verified and Statement 4 is implied by Lemma~\ref{l1} and Statement 3.
	
	Suppose that $\bu\not=\bu^*$ is a local minimum. Note that if $u_i = 0$ for some $i$, Lemma~\ref{l1} implies that $\bu[c] = 0$ for the $c^{\text{th}}$ component that includes node $i$. However, a strictly positive perturbation of $\bu[c]$ decreases the objective function and, therefore, $\bu$ cannot be a local minimum. Hence, it is enough to consider the case $\bu>0$. We show that $\bu$ cannot be D-stationary. This immediately certifies the validity of the first three statements. First, we prove that 
	\begin{equation}\label{eq5}
	\min_{k\in\Omega_i}\frac{u^*_k}{u_k}\leq \frac{u_i}{u^*_i}\leq\max_{k\in\Omega_i}\frac{u^*_k}{u_k} 
	\end{equation}
	for every $i\in\{1,\cdots, n\}$, where $\Omega_i = \{j|(i,j)\in\Omega\}$. By contradiction and without loss of generality, suppose that ${u_i}/{u^*_i}> \max_{k\in\Omega_i}{u^*_k}/{u_k}$ for some $i$. This implies that $u_iu_j>u^*_iu^*_j$ for every $j\in\Omega_i$. Therefore, a negative or positive perturbation of $u_i$ results in respective negative or positive directional derivatives, contradicting the D-stationarity of $\bu$. With no loss of generality, assume that the sparsity graph $\mathcal{G}(\Omega)$ is connected (since the arguments made in the sequel can be readily applied to every disjoint component of $\mathcal{G}(\Omega)$) and that the following ordering holds:
	\begin{equation}\label{eq17}
	0<\frac{u^*_1}{u_1}\leq \frac{u^*_2}{u_2}\leq \cdots\leq \frac{u^*_n}{u_n}
	\end{equation}
	Therefore, due to~\eqref{eq5}, we have
	\begin{equation}\label{eq7}
	0<\frac{u^*_1}{u_1}\leq \min_{k\in\Omega_i}\frac{u^*_k}{u_k}\leq \frac{u_i}{u^*_i}\leq\max_{k\in\Omega_i}\frac{u^*_k}{u_k}\leq \frac{u^*_n}{u_n}
	\end{equation}
	for every $i\in\{1,\cdots, n\}$.
	
	Since $\bu \not=\bu^*$, there exists some index $t$ such that $u_t\not=u^*_t$. This implies that $u^*_n/u_n>1$; otherwise, we should have $u^*_n/u_n\leq 1$. This together with~\eqref{eq17}, implies that $u^*_t/u_t<1$ and $u_t/u^*_t>1$, which contradicts~\eqref{eq7}.
	Now, define the sets
	\begin{align}
	& T_1 = \left\{i|\frac{u^*_i}{u_i} = \frac{u^*_n}{u_n}, 1\leq i\leq n\right\}\label{eq8}\\
	& T_2 = \left\{j|\frac{u_j}{u^*_j} = \frac{u^*_n}{u_n}, 1\leq j\leq n\right\}\label{eq9}
	\end{align}
	Moreover, define the set $N = V\backslash(T_1\cup T_2)$ and let $\bd$ be
	\begin{equation}\label{eqd}
	{d}_i= \left\{
	\begin{array}{ll}
	\frac{u_i}{u_n} &\text{if}\ i\in T_1\\
	-\frac{u_i}{u_n} &\text{if}\ i\in T_2\\
	0 &\text{if}\ i\in N\\
	\end{array} 
	\right.
	\end{equation} 
	Define a perturbation of $\bu$ as $\hat{\bu} = \bu+\bd\epsilon$ where $\epsilon>0$ is chosen to be sufficiently small.
	Next, the effect of the above perturbation on different terms of~\eqref{p2-sym} will be analyzed. To this goal, we divide $\Omega$ into four sets
	\begin{itemize}
		\item[1.] $(i,j)\in\Omega$ and $i,j\in T_1$: In this case, since $u_i<u^*_i$ and $u_j<u^*_j$, one can write
		\begin{align}
		|\hat{u}_i\hat{u}_j-u^*_iu^*_j| = u^*_iu^*_j - \hat{u}_i\hat{u}_j &= u^*_iu^*_j - \left(u_i\!+\!\frac{u_i}{u_n}\epsilon\right)\left(u_j\!+\!\frac{u_j}{u_n}\epsilon\right) \nonumber\\
		&= |u_iu_j-u^*_iu^*_j| -\left(\frac{2u_iu_j}{u_n}\right)\epsilon - \left(\frac{u_iu_j}{u_n^2}\right)\epsilon^2
		\end{align}
		where we have used the assumption $\bu^*, \bu>0$. 
		\item[2.] $(i,j)\in\Omega$ and $i,j\in T_2$: In this case, since $u_i>u^*_i$ and $u_j>u^*_j$, one can write
		\begin{align}
		|\hat{u}_i\hat{u}_j-u^*_iu^*_j| = \hat{u}_i\hat{u}_j-u^*_iu^*_j &= \left(u_i\!-\!\frac{u_i}{u_n}\epsilon\right)\left(u_j\!-\!\frac{u_j}{u_n}\epsilon\right) \!- u^*_iu^*_j\nonumber\\
		&= |u_iu_j-u^*_iu^*_j| -\left(\frac{2u_iu_j}{u_n}\right)\!\epsilon + \left(\frac{u_iu_j}{u_n^2}\right)\!\epsilon^2
		\end{align}
		where we have used the assumption $\bu^*, \bu>0$. 
		\item[3.] $(i,j)\in\Omega$, $i\in N$, and $j\in T_1\cup T_2$: According to the definitions of $T_1$ and $T_2$, we have
		\begin{equation}
		\frac{u_i}{u^*_i}<\frac{u^*_n}{u_n},\qquad \frac{u^*_i}{u_i}<\frac{u^*_n}{u_n}
		\end{equation}
		Now, if $j\in T_1$, one can write 
		\begin{equation}
		\frac{u_i}{u^*_i}<\frac{u^*_j}{u_j}\implies u_iu_j<u^*_iu^*_j
		\end{equation} 
		which implies that
		\begin{equation}
		|\hat{u}_i\hat{u}_j-u^*_iu^*_j| = u^*_iu^*_j - \hat{u}_i\hat{u}_j = u^*_iu^*_j - u_i\left(u_j+\frac{u_j}{u_n}\epsilon\right)= |{u}_i{u}_j-u^*_iu^*_j|-\left(\frac{u_iu_j}{u_n}\right)\!\epsilon
		\end{equation}
		Similarly, if $j\in T_2$, one can verify that
		\begin{equation}
		|\hat{u}_i\hat{u}_j-u^*_iu^*_j|= |{u}_i{u}_j-u^*_iu^*_j|-\left(\frac{u_iu_j}{u_n}\right)\!\epsilon
		\end{equation}
		\item[4.] $(i,j)\in\Omega$, $i\in T_1$, and $j\in T_2$: In this case, note that
		\begin{equation}
		|\hat{u}_i\hat{u}_j-u^*_iu^*_j| = \left|\left(u_i+\frac{u_i}{u_n}\epsilon\right)\left(u_j-\frac{u_j}{u_n}\epsilon\right)-u^*_iu^*_j\right|\leq |{u}_i{u}_j-u^*_iu^*_j|+\left(\frac{u_iu_j}{u_n^2}\right)\epsilon^2
		\end{equation}
	\end{itemize}
	
	The above analysis entails that---unless $N$ and the subgraphs of $\mathcal{G}(\Omega)$ induced by the nodes in $T_1$ or $T_2$ are empty---$f'(\bu,\bd)>0$ and $f'(\bu,-\bd)<0$, implying that $\bu$ cannot be D-stationary. On the other hand, these conditions enforce $\mathcal{G}(\Omega)$ to be bipartite, which is a contradiction. This completes the proof.~\hfill$\blacksquare$
	
	
	Next, we show that $\bu^*>0$ is \textit{almost} necessary to guarantee the absence of spurious local minima for~\eqref{p2-sym}. 
	
	\begin{proposition}\label{prop1}
		{Assume that $\bu^*\geq 0$ and that $\bu^*\not=0$ with $u_i^* = 0$ for some $i$. Then, upon choosing $\Omega = \{1,\dots,n\}^2\backslash \{(i,i)\}$,~\eqref{p2-sym} has a spurious local minimum.}
	\end{proposition}

	\begin{proof}
		See Appendix~\ref{app_prop1}.
	\end{proof}
	
	{The above corollary shows that if $\bu^*$ is non-negative with at least one zero element, even in the almost perfect scenario where the set $\Omega$ includes all of the measurements except for one, it may not be free of spurious local minima.}
	The next corollary shows that the assumption on the absence of bipartite components in $\mathcal{G}(\Omega)$ is also necessary for the uniqueness of the global solution. 
	
	\begin{proposition}\label{prop2}
		Given any vector $\bu^*>0$ and set $\Omega$, suppose that $\mG(\Omega)$ has a bipartite component. Then, the global solution of~\eqref{p2-sym} is not unique.
	\end{proposition}
	
	\begin{proof}
		Without loss of generality, suppose that $\mathcal{G}(\Omega)$ is a connected bipartite graph. For any vector $\bu^*>0$, the solution $\bu = \bu^*$ is globally optimal for~\eqref{p2-sym}. Suppose that the bipartite graph $\mathcal{G}(\Omega)$ partitions the entries of $\bu$ into two sets $V_1$ and $V_2$ such that $u_n\in V_1$. Based on some simple algebra, one can easily verify that, for a sufficiently small $\epsilon>0$, the solution
		\begin{equation}
		\hat{u}_i\leftarrow \left\{
		\begin{array}{ll}
		u_i+\frac{u_i}{u_n}\epsilon &\text{if}\ i\in V_1\\
		u_i-\frac{u_i}{u_n+\epsilon}\epsilon &\text{if}\ i\in V_2\\
		\end{array} 
		\right.
		\end{equation}
		is also globally optimal for~\eqref{p2-sym}.
	\end{proof}
	
	{\begin{remark}
			Suppose that $\mathbf{u}^*$ is a globally optimal solution of~\eqref{p2-sym} and that $\mathcal{G}(G)$ includes a bipartite component. Then, according to Proposition~\ref{prop2}, the part of $\mathbf{u}^*$ whose elements correspond to the nodes in this bipartite component can be \textit{perturbed} to attain another globally optimal solution, thereby resulting in the \textbf{non-uniqueness of the global solution}.
			On the other hand, the connectedness assumption is required to eliminate the undesirable stationary points on the \textit{boundary} of the feasible region. Roughly speaking, the elements of the vector variable $\mathbf{u}$ corresponding to different disconnected components can behave independently from each other, giving rise to spurious D-stationary points in the problem. To elaborate, recall that $\mathbf{u}[c]$ is a sub-vector of $\mathbf{u}$ induced by the $c^{\text{th}}$ component of $\mathcal{G}(G)$. Based on Lemma~\ref{l1}, the D-stationary points restricted to each disjoint component of $\mathcal{G}(G)$ are either strictly positive or equal to zero. Therefore, upon having two disconnected components $c_1$ and $c_2$, the points $\mathbf{u}' = \begin{bmatrix}
			{\mathbf{u}^*[c_1]}^\top & 0
			\end{bmatrix}^\top$ and $\mathbf{u}'' = \begin{bmatrix}
			0 & {\mathbf{u}^*[c_2]}^\top
			\end{bmatrix}^\top$ are indeed D-stationary points of~\eqref{snpca}, thereby resulting in \textbf{spurious stationary points}.
	\end{remark}}
	
	\noindent{\bf Asymmetric case:} Next, we consider~\eqref{p2-asym} in the noiseless scenario by analyzing its symmetrized counterpart~\eqref{p2_asym_sym}. Based on the construction of $\bar{\Omega}$, the corresponding sparsity graph $\mathcal{G}(\bar{\Omega})$ is bipartite. On the other hand, according to Proposition~\ref{prop2}, the existence of a bipartite component in $\mathcal{G}(\bar\Omega)$ makes a part of the solution~\textit{invariant to scaling}, which subsequently results in the non-uniqueness of the global minimum. The additional regularization term in~\eqref{p2_asym_sym} is introduced to circumvent this issue by penalizing the difference in the norms of $\bu$ and $\bv$.
	
	\begin{theorem}\label{thm1_asym}
		Suppose that $\bw^*>0$ and $\mathcal{G}(\bar\Omega)$ is connected. Then, the following statements hold for \eqref{p2_asym_sym}:
		\begin{itemize}
			\item[1.] The points $\bw = 0$ and $\bw$ with the properties $\bw\bw^\top = \bw^*{\bw^*}^\top$ and $\sum_{i=1}^{m}w^{2}_i=\sum_{j = m+1}^{m+n}w^{2}_j$ are the only D-min-stationary points;
			\item[2.] The point $\bw = 0$ is a local maximum;
			\item[3.] In the positive orthant, the point $\bw$ with the properties $\bw\bw^\top = \bw^*{\bw^*}^\top$ and $\sum_{i=1}^{m}w^{2}_i=\sum_{j = m+1}^{m+n}w^{2}_j$ is the only D-stationary point.
		\end{itemize}
	\end{theorem}
	
	\begin{proof}
		See Appendix~\ref{app_thm1_asym}.
	\end{proof}
	
	\begin{remark}
		Notice that, unlike the symmetric case, Theorem~\ref{thm1_asym} requires the connectedness of $\mG(\bar\Omega)$. This is due to the additional regularization term in~\eqref{anpca}. In particular, similar arguments do not necessarily hold for the disjoint components of $\mG(\bar\Omega)$ because of the coupling nature of the regularization term.
	\end{remark}
	
	
	\subsection{Probabilistic Guarantees}
	
	Next, we consider the random sampling regime. Similar to the previous subsection, we first focus on the symmetric case. 
	
	\vspace{2mm}
	\noindent{\bf Symmetric case:} Suppose that every element of the upper triangular part of the matrix $\bu^*{\bu^*}^\top$ is measured independently with probability $p$. In other words, for every $(i,j)\in\{1,2,...,n\}^2$ and $i\leq j$, the probability of $(i,j)$ belonging to $\Omega$ is equal to $p$. 
	
	{\begin{theorem}\label{thm2}
			Suppose that $n\geq 2$, $\bu^*>0$, and $p\geq \min\left\{1,\frac{(2\eta+2)\log n + 2}{n-1}\right\}$ for some constant $\eta\geq 1$. Then, the following statements hold for \eqref{snpca} with probability of at least $1-\frac{3}{2}n^{-\eta}$:
			\begin{itemize}
				\item[1.] The points $\bu = \bu ^*$ and $\bu = 0$ are the only D-min-stationary points;
				\item[2.] The point $\bu = 0$ is a local maximum;
				\item[3.] In the positive orthant, the point $\bu = \bu ^*$ is the only D-stationary point.
			\end{itemize}
	\end{theorem}}
	
	{Before presenting the proof of Theorem~\ref{thm2}, we note that the required lower bound on $p$ is to guarantee that the random graph $\mathcal{G}(\Omega)$ is connected with high probability. This implies that Theorem~\ref{thm1} can be invoked to verify the statements of Theorem~\ref{thm2}. It is worthwhile to mention that the classical results on \textit{Erd\"os-R\'enyi} graphs characterize the \textit{asymptotic} properties of $\mathcal{G}(\Omega)$ as $n$ approaches infinity. In particular, it is shown by~\cite{erdds1959random} that with the choice of $p = \frac{\log n+c}{n}$ for some $c>0$, $\mathcal{G}(\Omega)$ becomes connected with probability of at least $\Omega(e^{-e^{-c}})$ as $n\to\infty$. In contrast, we introduce the following non-asymptotic result characterizing the probability that $\mathcal{G}(\Omega)$ is connected and non-bipartite for any finite $n\geq 2$, and subsequently use it to prove Theorem~\ref{thm2}.

		\begin{lemma}\label{l2}
			Given a constant $\eta\geq 1$, suppose that $p\geq \min\left\{1,\frac{(2\eta+2)\log n + 2}{n-1}\right\}$  and $n\geq 2$. Then, $\mG(\Omega)$ is connected and non-bipartite with probability of at least $1-\frac{3}{2}n^{-\eta}$.
		\end{lemma}
		
		\begin{proof}
			See Appendix~\ref{app_l2}.
		\end{proof}

		\noindent\textbf{Proof of Theorem~\ref{thm2}:} The proof immediately follows from Theorem~\ref{thm1} and Lemma~\ref{l2}.~\hfill$\blacksquare$}
	
	\vspace{2mm}
	
	Similar to the deterministic case, we will show that both assumptions $\bu^*>0$ and $p\gtrsim \log n/n$ are \textit{almost} necessary for the successful recovery of the global solution of~\eqref{p2-sym}. In particular, it will be proven that relaxing $\bu^*>0$ to $\bu^*\geq 0$ will result in an instance that possesses a spurious local solution with non-negligible probability. Furthermore, it will be shown that the choice $p \approx \log n/n$ is optimal\textemdash modulo $\log n$-factor\textemdash for the unique recovery of the global solution.
	
	\begin{proposition}\label{prop3}
		Assuming that $\bu^*\geq 0$ with $u^*_i = 0$ for some $i\in\{1,\dots,n\}$ and that $p<1$,~\eqref{p2-sym} has a spurious local minimum with probability of at least $1-p>0$.
	\end{proposition}
	
	\begin{proof}
		Suppose that $\bu^*\geq 0$ and there exists an index $i$ such that $u^*_i = 0$. The proof of Proposition~\ref{prop1} can be used to show that excluding the measurement $(i,i)$ gives rise to a spurious local minimum. This occurs with probability $1-p$. The details are omitted due to their similarities to the proof of Proposition~\ref{prop1}.
	\end{proof}
	
	\begin{proposition}\label{prop4}
		Given any $\bu^*>0$, suppose that $np\rightarrow 0$ as $n\rightarrow\infty$. Then, the global solution of~\eqref{p2-sym} is not unique with probability approaching to one.
	\end{proposition}
	
	\begin{proof}
		See Appendix~\ref{app_prop4}.
	\end{proof}
	
	
	\noindent{\bf Asymmetric case:} Consider~\eqref{p2-asym} under a random sampling regime, where each element of $\bu^*{\bv^*}^\top$ is independently observed with probability $p$. Next, the analog of Theorem~\ref{thm2} for the asymmetric case is provided.
	
	{\begin{theorem}\label{thm2_asym}
			Suppose that $n,m\geq 2$, $\bw^*>0$, and $p\geq \min\left\{1,\frac{(m+n)((1+\eta)\log(mn)+1)}{(m-1)(n-1)}\right\}$ for some constant $\eta\geq 1$. Then, the following statements hold for \eqref{p2_asym_sym} with probability of at least $1-2(mn)^{-\eta}-4(mn)^{-2\eta}$:
			\begin{itemize}
				\item[1.] The points $\bw = 0$ and $\bw$ with the properties $\bw\bw^\top = \bw^*{\bw^*}^\top$ and $\sum_{i=1}^{m}w^{2}_i=\sum_{j = m+1}^{m+n}w^{2}_j$ are the only D-min-stationary points;
				\item[2.] The point $\bw = 0$ is a local maximum;
				\item[3.] In the positive orthant, the point $\bw$ with the properties $\bw\bw^\top = \bw^*{\bw^*}^\top$ and $\sum_{i=1}^{m}w^{2}_i=\sum_{j = m+1}^{m+n}w^{2}_j$ is the only D-stationary point.
			\end{itemize}
		\end{theorem}
		
		
		Before presenting the proof of Theorem~\ref{thm2_asym}, we note that $\mG(\bar{\Omega})$ no longer corresponds to an {Erd\"os-R\'enyi} random graph due to its bipartite structure. Therefore, we present the analog of Lemma~\ref{l2} for random bipartite graphs.
		\begin{lemma}\label{l_bipartite}
			Given a constant $\eta\geq 1$, suppose that $p\geq \min\left\{1,\frac{(m+n)((1+\eta)\log(mn)+1)}{(m-1)(n-1)}\right\}$ and $m,n\geq 2$. Then, $\mathcal{G}(\bar{\Omega})$ is connected with probability of at least $1-2(mn)^{-\eta}-4(mn)^{-2\eta}$.
		\end{lemma}
		\begin{proof}
			See Appendix~\ref{app_l_bipartite}.
		\end{proof}
		
		%
		
		\vspace{2mm}
		\noindent{\bf Proof of Theorem~\ref{thm2_asym}:} The proof immediately follows from Theorem~\ref{thm1_asym} and Lemma~\ref{l_bipartite}.~\hfill$\blacksquare$
		
		Before proceeding, we note that, similar to the classical results on the \textit{Erd\"os-R\'enyi} graphs, there are asymptotic results guaranteeing the connectedness of a random bipartite graph as a function of $p$. In particular,~\cite{saltykov1995number} shows that $\mathcal{G}(\bar{\Omega})$ is connected with probability approaching to 1 as $m+n\to \infty$, provided that $p \geq 3\left(1+\frac{m}{n}\right)^{-1}\frac{(n+m)\log(n+m)}{nm}$. Lemma~\ref{l_bipartite} offers another lower bound on $p$ that matches this threshold (modulo a constant factor), while being non-asymptotic in nature. In particular, it characterizes the probability that the random bipartite graph is connected for \textit{all $m,n\geq 2$}.
	}
	
	
	\section{Extension to Noisy Positive RPCA}\label{sec6}
	
	In this section, we will show that an additive sparse noise with arbitrary values does not drastically change the landscape of the RPCA. In other words, a limited number of grossly wrong measurements will not introduce any spurious local solution to the positive RPCA. The key idea is to prove that the direction of descent that was introduced in the previous section is also valid when the measurements are not perfect, i.e., when they are subject to sparse noise. To this goal, consider the following problem in the symmetric case:
	{\begin{equation}\label{p3_sym}
		\min_{\bu\geq 0}\quad \underbrace{\sum_{(i,j)\in\Omega}|u_iu_j-X_{ij}|}_{f(\bu)}
		\end{equation}}
	where 
	\begin{equation}
	X = \bu^*{\bu^*}^\top+S
	\end{equation}
	is the matrix of true measurements perturbed with sparse noise. Similarly, consider the following problem for the asymmetric case:
	\begin{equation}\label{p3_aasym}
	\min_{\bu\geq 0, \bv\geq 0}\quad \sum_{(i,j)\in\Omega}|u_iv_j-{X}_{ij}|+\alpha\left|\sum_{i=1}^{m}u_i^2-\sum_{j = 1}^{n}v_j^2\right|
	\end{equation}
	where $\alpha$ is an arbitrary positive number. After symmetrization,~\eqref{p3_aasym} can be re-written as
	\begin{equation}\label{p3_asym}
	\min_{\bw\geq 0}\quad \underbrace{\sum_{(i,j)\in\bar\Omega}|w_iw_j-\bar{X}_{ij}|+\alpha\left|\sum_{i=1}^{m}w_i^2-\sum_{j = m+1}^{m+n}w_j^2\right|}_{f(\bw)}
	\end{equation}
	where 
	\begin{equation}\label{Xbar}
	\bar{X} = \bw\bw^\top+\bar{S}
	\end{equation}
	for $\bar{X}\in\mathbb{R}^{(n+m)\times (n+m)}$ 
	and 
	\begin{equation}\label{Sbar}
	\bar{S} = \begin{bmatrix}
	0 & S\\
	S^\top & 0
	\end{bmatrix}
	\end{equation}
	Furthermore, define $\bar{B} = \{(i,j): (i,j)\in\bar\Omega, \bar{S}_{ij}\not=0\}$ and $\bar{G} = \{(i,j): (i,j)\in\bar\Omega, \bar{S}_{ij}=0\}$ as the sets of bad and good measurements for the symmetrized problem, respectively.
	In this work, we do not impose any assumption on the maximum value of the nonzero elements of $S$. However, without loss of generality, one may assume that $\bu^*{\bu^*}^\top+S > 0$ and $\bw^*{\bw^*}^\top+\bar{S} > 0$; otherwise, the non-positive elements can be discarded due to the assumptions $\bu^*>0$ and $(\bu^*, \bv^*)>0$. In fact, we impose a slightly more stronger condition in this work.
	
	\begin{assumption}\label{assum1}
		There exists a constant $c\in (0,1]$ such that $S_{ij}+u^*_iu^*_j>cu^{*^2}_{\min}$ and $\bar{S}_{ij}+w^*_iw^*_j>cw^{*^2}_{\min}$ for~\eqref{p3_sym} and~\eqref{p3_asym}, respectively.
	\end{assumption}
	
	\subsection{Identifiability}
	
	Intuitively, the non-negative RPCA under the unknown-but-sparse noise is more challenging to solve than its noiseless counterpart. In particular, one may consider~\eqref{p3_sym} as a variant of~\eqref{p2-sym} discussed in the previous section, where the locations of the bad measurements are unknown; if these locations were known, they could have been discarded to reduce the problem to~\eqref{p2-sym}. If the measurements are subject to unknown noise, one of the main issues arises from the identifiability of the solution.
	To further elaborate, we will offer an example below.
	
	\begin{example}
		Suppose that $X(\epsilon) = (e_1+\mathbf{1}\epsilon)(e_1+\mathbf{1}\epsilon)^\top$, where $e_1$ is the first unit vector and $\mathbf{1}$ is a vector of ones. Assuming that $\Omega = \{1,...,n\}^2$, one can decompose $X(\epsilon)$ in two forms
		\begin{subequations}
			\begin{align}
			& X(\epsilon) = \underbrace{(e_1+\mathbf{1}\epsilon)(e_1+\mathbf{1}\epsilon)^\top}_{\bu^*_1{\bu^*_1}^\top}+\underbrace{0}_{S_1}\\
			& X(\epsilon) = \underbrace{\mathbf{1}\mathbf{1}^\top\epsilon^2}_{\bu^*_2{\bu^*_2}^\top}+\underbrace{e_1e_1^\top+\mathbf{1}e_1^\top\epsilon+e_1\mathbf{1}^\top\epsilon}_{S_2}
			\end{align}
		\end{subequations}
		For every $\epsilon>0$, both $S_1$ and $S_2$ can be considered as sparse matrices since the number of nonzero elements in each of these matrices is at most on the order of $O(n)$. However, unless more restrictions on the number of nonzero elements at each row or column of $S$ are imposed, it is impossible to distinguish between these two cases. This implies that the solution is not identifiable. \\
	\end{example}
	
	In order to ensure that the solution is identifiable in the symmetric case, we assume that $\Delta(\mG(B))\leq \eta\cdot \delta(\mG(G))$ for some constant $\eta\leq 1$ to be defined later. Roughly speaking, this implies that at each row of the measurement matrix, the number of good measurements should be at least as large as the number of bad ones. 
	Similar to the work by~\cite{ge2016matrix, ge2017no}, we consider the regularized version of the problem, as in
	\begin{equation}\tag{P2-Sym}\label{p3_sym_reg}
	\min_{\bu\geq 0}\quad \underbrace{\sum_{(i,j)\in\Omega}|u_iu_j-X_{ij}|+ R(\bu)}_{f_{\mathrm{reg}}(\bu)}
	\end{equation}
	where $R(\bu)$ is a regularizer defined as 
	\begin{equation}
	R(\bu) = \lambda\sum_{i = 1}^{n}\left(u_i-\beta\right)^4\mathbb{I}_{u_i\geq\beta}
	\end{equation}
	for some fixed parameters $\lambda$ and $\beta$ to be specified later. Similarly, one can define an analogous regularization for~\eqref{p3_asym} as 
	\begin{equation}\tag{P2-Asym}\label{p3_asym_reg}
	\min_{\bw\geq 0}\quad \underbrace{\sum_{(i,j)\in\bar\Omega}|w_iw_j-\bar{X}_{ij}|+\alpha\left|\sum_{i=1}^{m}w_i^2-\sum_{j = m+1}^{m+n}w_j^2\right|+R(\bw)}_{f_{\mathrm{reg}}(\bw)}
	\end{equation}
	with 
	\begin{equation}
	R(\bu) = \lambda\sum_{i = 1}^{m+n}\left(w_i-\beta\right)^4\mathbb{I}_{w_i\geq\beta}
	\end{equation}
	for some fixed parameters $\lambda$ and $\beta$ to be specified later. Note that the defined regularization function is convex in its domain. In particular, it eliminates the candidate solutions that are far from the true solution. Without loss of generality and to streamline the presentation, it is assumed that $u^*_{\max} = w^*_{\max} = 1$ in the sequel. 
	%
	%
	\begin{lemma}\label{l4}
		Consider the parameter $c$ defined in Assumption~\ref{assum1}. The following statements hold:
		\begin{itemize}
			\item[-] By choosing $\beta = 1$ and $\lambda = n/2$, any D-stationary point $\bu>0$ of~\eqref{p3_sym_reg} satisfies the inequalities $(c/2)u^{*^2}_{\min}\leq u_{\min}\leq u_{\max}\leq 2$.
			\item[-] By choosing $\beta = 1$ and $\lambda = (m+n)/2$, any D-stationary point $\bw>0$ of~\eqref{p3_asym_reg} satisfies the inequalities $(c/2)w^{*^2}_{\min}\leq w_{\min}\leq w_{\max}\leq 2$.
		\end{itemize}
	\end{lemma}
	
	\begin{proof}
		See Appendix~\ref{app_l4}.
	\end{proof}
	
	\subsection{Deterministic Guarantees}
	
	In what follows, the deterministic conditions under which~\eqref{p3_sym_reg} and~\eqref{p3_asym_reg} have benign landscape will be investigated. The results of this subsection will be the building blocks for the derivation of the main theorems for both symmetric and asymmetric positive RPCA under the random sampling and noise regime. Note that the analysis of the landscape will be more involved in this case since the effect of the regularizer should be taken into account.
	
	\vspace{2mm}
	\noindent{\bf Symmetric case:} Recall that, for the sparsity graph $\mG(\Omega)$, $\Delta(\mG(\Omega))$ and $\delta(\mG(\Omega))$ correspond to its maximum and minimum degrees, respectively.
	
	\begin{theorem}\label{thm4}
		Suppose that 
		\begin{itemize}
			\item[i.] $\bu^*>0$;
			\item[ii.]
			${\delta(\mG(G))}>({48/c^2}){\kappa(\bu^*)^4}{\Delta(\mG(B))}$;
			\item[iii.] $\mathcal{G}(\Omega)$ has no bipartite component.
		\end{itemize}
		Then, with the choice of $\beta = 1$ and $\lambda = n/2$ for the parameters of the regularization function $R(\bu)$,
		the following statements hold for \eqref{p3_sym_reg}:
		\begin{itemize}
			\item[1.] It does not have any spurious local minimum;
			\item[2.] The point $\bu = \bu^*$ is the unique global minimum;
			\item[3.] In the positive orthant, the point $\bu = \bu^*$ is the only D-stationary point.
		\end{itemize}
		Additionally, if $\mathcal{G}(\Omega)$ is connected, the following statements hold for \eqref{p3_sym_reg}:
		\begin{itemize}
			\item[4.] The points $\bu = \bu^*$ and $\bu = 0$ are the only D-min-stationary points;
			\item[5.] The point $\bu = 0$ is a local maximum.
		\end{itemize}
	\end{theorem}
	
	\begin{proof}
		See Appendix~\ref{app_thm4}.
	\end{proof}
	
	\vspace{2mm}
	\noindent{\bf Asymmetric case:} Theorem~\ref{thm4} has the following natural extension to asymmetric problems.
	
	\begin{theorem}\label{thm4_asym}
		Suppose that 
		\begin{itemize}
			\item[i.] $\bw^*>0$;
			\item[ii.] ${\delta(\mG(\bar G))}>({48}/c^2){\kappa(\bw^*)^4}{\Delta(\mG(\bar B))}$;
			\item[iii.] $\mathcal{G}(\bar G)$ is connected.
		\end{itemize}
		Then, with the choice of $\beta = 1$ and $\lambda = (m+n)/2$ for the parameters of the regularization function $R(\bw)$, the following statements hold for \eqref{p3_asym_reg}:
		\begin{itemize}
			\item[1.] The points $\bw = 0$ and $\bw$ with the properties $\bw\bw^\top = \bw^*{\bw^*}^\top$ and $\sum_{i=1}^{m}w^{2}_i=\sum_{j = m+1}^{m+n}w^{2}_j$ are the only D-min-stationary points;
			\item[2.] The point $\bw = 0$ is a local maximum;
			\item[3.] In the positive orthant, the point $\bw$ with the properties $\bw\bw^\top = \bw^*{\bw^*}^\top$ and $\sum_{i=1}^{m}w^{2}_i=\sum_{j = m+1}^{m+n}w^{2}_j$ is the only D-stationary point.
		\end{itemize}
	\end{theorem}
	
	\begin{proof}
		The proof is omitted due to its similarity to that of Theorem~\ref{thm4}.
	\end{proof}
	\subsection{Probabilistic Guarantees}
	
	As an extension to our previous results, we analyze the landscape of the noisy non-negative RPCA with randomness both in the location of the samples and in the structure of the noise matrix. Suppose that for the symmetric case,  with probability $d$, each element of the upper triangular part of $X$ is independently corrupted with an arbitrary noise value. In other words, for every $(i,j)$ with $i\leq j$, one can write
	\begin{equation}
	X_{ij} = \left\{
	\begin{array}{ll}
	u^*_iu^*_j& \text{with probability}\ 1-d\\
	\text{arbitrary}& \text{with probability}\ d
	\end{array} 
	\right.
	\end{equation}
	Furthermore, similar to the preceding section, suppose that every element of the upper triangular part of $X = \bu^*{\bu^*}^\top+S$ is independently measured with probability $p$. The randomness in the location of the measurements and noise is naturally extended to the asymmetric case by considering the symmetrized $\bar{X}$ and $\bar{S}$ defined in~\eqref{Xbar} and~\eqref{Sbar}, respectively.
	
	\vspace{2mm}
	\noindent{\bf Symmetric case:} First, the main result in the symmetric case is presented below.
	\begin{theorem}\label{thm5}
		{Suppose that 
			\begin{itemize}
				\item[i.] $n\geq 2$,
				\item[ii.] $\bu^*>0$,
				\item[iii.] $d<\frac{1}{(144/c^2)k(\bu^*)^4+1}$,
				\item[iv.] $p>\frac{(1740/c^2)\kappa(\bu^*)^4(1+\eta)\log n}{n}$,
			\end{itemize}
			for some $\eta> 0$. Then, with the choice of $\beta = 1$ and $\lambda = n/2$ for the parameters of the regularization function $R(\bu)$,
			the following statements hold for \eqref{p3_sym_reg} with probability of at least $1-3n^{-\eta}$:
			\begin{itemize}
				\item[1.] The points $\bu = \bu ^*$ and $\bu = 0$ are the only D-min-stationary points;
				\item[2.] The point $\bu = 0$ is a local maximum;
				\item[3.] In the positive orthant, the point $\bu = \bu ^*$ is the only D-stationary point.
		\end{itemize}}
	\end{theorem}
	
	
	To prove Theorem~\ref{thm5}, first we present the following lemma on the concentration of the minimum and maximum degrees of random graphs.
	
	\begin{lemma}\label{l7}
		{Consider a random graph $\mG(n,p)$. Given a constant $\eta>0$, the inequality:
			\begin{align}
			&\mathbb{P}\left(\Delta(\mG(n,p))\geq\max\left\{\frac{3np}{2},18(1+\eta)\log n\right\}\right)\leq n^{-\eta}\label{eq72}
			\end{align}
			holds for every $0< p\leq 1$. Furthermore, we have
			\begin{align}
			\mathbb{P}\left(\delta(\mG(n,p))\leq\frac{np}{2}\right)\leq n^{-\eta}
			\end{align}
			provided that $p\geq \frac{12(1+\eta)\log n}{n}$.}
	\end{lemma}
	\begin{proof}
		See Appendix~\ref{app_l7}.
	\end{proof}
	
	\begin{remark}
		Note that since the degree of each node in $\mG(n,p)$ is concentrated around $np$ with high probability, one may speculate that $\Delta(\mG(n,p))$ and $\delta(\mG(n,p))$ should also concentrate around $np$ for all values of $p$ and hence the inclusion of $18(1+\eta)\log n$ in~\eqref{eq72} may seem redundant. Surprisingly, this is not the case in general. In fact, it can be shown that if $p = 1/n$ (and hence $np=1$), there exists a node whose degree is lower bounded by ${\log n}/{\log\log n}$ with high probability. This explains the reasoning behind the inclusion of $18(1+\eta)\log n$ in the lemma.
	\end{remark}
	
	\noindent\textbf{Proof of Theorem~\ref{thm5}:} {In light of Lemma~\ref{l2}, the bounds on $p$ and $d$ guarantee that $\mG(G)$ is connected and non-bipartite with probability of at least $1-\frac{3}{2}n^{-430(1+\eta)}$.
		Therefore, the proof is completed by invoking Theorem~\ref{thm4}, provided that the second condition of Theorem~\ref{thm4} holds. Define the events $\mathcal{E}_1 = \left\{\Delta(\mG(B))\leq\max\left\{\frac{3npd}{2}, 18(1+\eta)\log n\right\}\right\}$ and $\mathcal{E}_2 = \left\{\delta(\mG(G))\geq\frac{np(1-d)}{2}\right\}$. Observe that Lemma~\ref{l7} together with the bounds on $p$ and $d$ results in the inequalities
		\begin{subequations}
			\begin{align}
			& \mathbb{P}\left(\mathcal{E}_1\right)\geq 1-n^{-\eta}\label{eq81}\\
			& \mathbb{P}\left(\mathcal{E}_2\right)\geq 1-n^{-144\eta}\label{eq82}
			\end{align}
		\end{subequations}
		This in turn implies that the events $\mathcal{E}_1$ and $\mathcal{E}_2$ occur with probability of at least $1-n^{-\eta}-n^{-144\eta}$. Conditioned on these events, it suffices to show that
		\begin{equation}\label{eqvalid}
		\frac{np(1-d)}{2}>\frac{48}{c^2}\kappa(\bu^*)^4\max\left\{\frac{3npd}{2}, 18(1+\eta)\log n\right\}
		\end{equation}
		in order to certify the validity of the second condition of Theorem~\ref{thm4}. It can be easily verified that the assumed upper and lower bounds on $p$ and $d$ guarantee the validity of~\eqref{eqvalid}. Therefore, a simple union bound and the fact that $n^{-\eta}>\frac{3}{2}n^{-430(1+\eta)}$ imply that the conditions of Theorem~\ref{thm4} are satisfied with probability of at least $1-3n^{-\eta}$.~\hfill$\blacksquare$}
	\vspace{2mm}
	
	A number of interesting corollaries can be derived based on Theorem~\ref{thm5}.
	
	\begin{corollary}\label{cor1}
		Suppose that $p$ is a positive number independent of $n$ and $d \lesssim \log n/n$. Then, under an appropriate choice of parameters for the regularization function, the statements of Theorem~\ref{thm5} hold  with overwhelming probability, provided that $\kappa(\bu^*) \lesssim ({n/\log n})^{1/4}$.
	\end{corollary}
	
	Corollary~\ref{cor1} implies that, roughly speaking, if the total number of measurements is sufficiently large (i.e., on the order of $n^2$), then up to factor of $n\log n$ bad measurements with arbitrary magnitudes will not introduce any spurious local solution to the problem. Under such circumstances, the required upper bound on the ratio between the maximum and the minimum entries of $\bu^*$ will be more relaxed as the dimension of the problem grows. 
	
	\begin{corollary}\label{cor2}
		Suppose that $p$ is a positive number independent of $n$ and that $d \lesssim n^{\epsilon-1}$ for some $\epsilon\in [0,1)$. Then, under an appropriate choice of parameters for the regularization function, the statements of Theorem~\ref{thm5} hold  with overwhelming probability, provided that $\kappa(\bu^*) \lesssim n^{(1-{\epsilon})/{4}}$.
	\end{corollary}
	
	Corollary~\ref{cor2} describes an interesting trade-off between the sparsity level of the noise and the maximum allowable variation in the entries of $\bu^*$; roughly speaking, as $\kappa(\bu^*)$ decreases, a larger number of noisy elements can be added to the problem without creating any spurious local minimum.
	The next corollary shows that a constant fraction of the measurements can be grossly corrupted without affecting the landscape of the problem, provided that $\kappa(\bu^*)$ is uniformly bounded from above.
	
	\begin{corollary}\label{cor3}
		Suppose that $p$ and $d$ are positive numbers independent of $n$ and that $d < \frac{1}{(144/c^2)+1}$. Then, under an appropriate choice of parameters for the regularization function, the statements of Theorem~\ref{thm5} hold with overwhelming probability, provided that $\kappa(\bu^*)\leq \left(\frac{1-d}{(144/c^2)d}\right)^{1/4}$. 
	\end{corollary}
	
	\noindent{\bf Asymmetric case:} The aforementioned results on the symmetric positive RPCA under random sampling and noise will be generalized to the asymmetric case below.
	
	\begin{theorem}\label{thm6}
		{Define $r = m/n$ and suppose that 
			\begin{itemize}
				\item[i.] $n\geq m\geq 2$,
				\item[ii.] $\bw^*>0$,
				\item[iii.] $d<\frac{r}{(144/c^2)\kappa(\bw^*)^4+r}$,
				\item[iv.] $p>\frac{(1740/c^2)\kappa(\bw^*)^4(1+\eta)n\log n}{m^2}$,
			\end{itemize}
			for some $\eta>0$. Then, with the choice of $\beta = 1$ and $\lambda = (m+n)/2$ for the parameters of the regularization function $R(\bu)$,
			the following statements hold for \eqref{p3_sym_reg} with probability of at least $1-10n^{-\eta}$:
			\begin{itemize}
				\item[1.] The points $\bw = 0$ and $\bw$ with the properties $\bw\bw^\top = \bw^*{\bw^*}^\top$ and $\sum_{i=1}^{m}w^{2}_i=\sum_{j = m+1}^{m+n}w^{2}_j$ are the only D-min-stationary points;
				\item[2.] The point $\bw = 0$ is a local maximum;
				\item[3.] In the positive orthant, the point $\bw$ with the properties $\bw\bw^\top = \bw^*{\bw^*}^\top$ and $\sum_{i=1}^{m}w^{2}_i=\sum_{j = m+1}^{m+n}w^{2}_j$ is the only D-stationary point.
		\end{itemize}}
	\end{theorem}
	
	To prove Theorem~\ref{thm6}, we derive a concentration bound on the minimum and maximum degree of the random bipartite graphs. Define $\mathcal{G}(m, n, p)$ as a bipartite graph with the vertex partitions $V_u = \{1,\cdots,m\}$ and $V_v = \{m+1, \cdots, m+n\}$ where each edge is independently included in the graph with probability $p$.
	
	\begin{lemma}\label{l10}
		{Consider a random bipartite graph $\mG(m,n,p)$. Given a constant $\eta>0$, the inequality
			\begin{align}\label{eq722}
			\mathbb{P}\left(\Delta(\mG(m,n,p))\geq\max\left\{\frac{3np}{2}, \frac{18(1+\eta)n\log n}{m}\right\}\right)\leq2n^{-\eta}
			\end{align}
			holds for every $0< p\leq 1$. Furthermore, we have
			\begin{align}
			\mathbb{P}\left(\delta(\mG(m,n,p))\leq\frac{mp}{2}\right)\leq 2n^{-\eta}
			\end{align}
			provided that $p\geq {12(1+\eta)\log n}/m$.}
	\end{lemma}
	
	\begin{proof}
		See Appendix~\ref{app_l10}.
	\end{proof}
	
	\noindent{\bf Proof of Theorem~\ref{thm6}:} {The bounds on $p$ and $d$ indeed guarantee that $\mG(\bar{G})$ is connected with overwhelming probability. Based on this fact, the result of Lemma~\ref{l10} and the proof of Theorem~\ref{thm5} can be combined to arrive at this theorem. The details are omitted for brevity.~\hfill$\blacksquare$
		
		\begin{remark}
			The presented probability guarantees for RPCA share some similarities with those derived for noisy matrix completion in~\cite{ge2017no, ge2016matrix}. In particular, according to Theorems~\ref{thm5} and~\ref{thm6} and similar to the results of~\cite{ge2017no, ge2016matrix}, the probability of having a spurious local solution decreases polynomially with respect to the dimension of the problem. Furthermore, similar to our work, the required lower bound on the sampling probability $p$ in~\cite{ge2017no, ge2016matrix} scales polynomially with respect to the condition number of the true solution. Finally, for non-symmetric noisy matrix completion problem,~\cite{ge2017no} shows that the required lower bound on $p$ scales as $\frac{\log n}{m}$. Comparing this dependency with the one introduced in Theorem~\ref{thm6}, it can be inferred that our proposed lower bound is higher by a factor of $\frac{n}{m}$; this is not surprising considering the fundamentally different natures of these problems. 
	\end{remark}}

	\vspace{2mm}
	
	\section{Global Convergence of Local Search Algorithms}\label{sec8}
	
	So far, it has been shown that the positive RPCA is free of spurious local minima. Furthermore, it has been proven that the global solution is the only D-stationary point in the positive orthant. The question of interest in this section is: How could this unique D-stationary point be obtained? Before answering this question, we will take a detour and revisit the notion of stationarity for smooth optimization problems. Recall that $\bar{\bx}$ is a stationary point of a differentiable function $f(\bx)$ if and only if $\nabla f(\bx) = 0$ and, under some mild conditions, basic local search algorithms will converge to a stationary point. Therefore, the uniqueness of the stationary point for a smooth optimization problem immediately implies the convergence to global solution. Extra caution should be taken when dealing with non-smooth optimization. In particular, the convergence of classical local search algorithms may fail to hold since the gradient and/or Hessian of the function may not exist at every iteration. To deal with this issue, different local search algorithms have been introduced to guarantee convergence to generalized notions of stationary points for non-smooth optimization, such as directional-stationary (which is used in this paper) or Clarke-stationary (to be defined next). 
	
	For a non-smooth and locally Lipschitz function $h(\bx)$ over the convex set $\mathcal{X}$, define the Clarke generalized directional derivative at the point $\bar\bx$ in the feasible direction $\bd$ as
	\begin{equation}
	h^\circ(\bx,\mathbf{d}) := \underset{\begin{subarray}{c}
		\by\rightarrow \bx\\
		t\downarrow 0
		\end{subarray}}{\lim\sup}\frac{h(\by+t\mathbf{d})-h(\by)}{t} 
	\end{equation}
	Note the difference between the ordinary directional derivative $h'(\bx,\mathbf{d})$ and its Clarke generalized counterpart: in the latter, the limit is taken with respect to a~\textit{variable} vector $\by$ that approaches $\bar{\bx}$, rather than taking the limit exactly at $\bar{\bx}$.
	The Clarke differential of $h(\bx)$ at $\bar{\bx}$ is defined as the following set~(\cite{clarke1990optimization}):
	\begin{equation}\label{partial}
	\partial_C h(\bar{\bx}) := \{\mathbf{\psi} | h^\circ(\bx,\mathbf{d})\geq\langle\mathbf{\psi}, \mathbf{d}\rangle, \forall\mathbf{d}\in\mathbb{R}^{n}\ \text{such that}\ \bx+\bd\in \mathcal{X}\}
	\end{equation} 
	where $\mathcal{X}$ is the feasible set of the problem. A point $\bar{\bx}$ is Clarke-stationary (or C-stationary) if $0\in\partial_C(\bar{\bx})$, or equivalently, $h^\circ(\bar\bx,\mathbf{d})\geq 0$ for every feasible direction $\bd$. It is well known that C-stationary is a weaker condition than the D-min-stationarity. In particular, every D-min-stationary point is C-stationary but not all C-stationary points are D-min-stationary. 
	
	\begin{sloppypar}
		On the other hand, although some local search algorithms converge to D-min-stationary points for problems with special structures~(\cite{cui2018composite}), the most well-known numerical algorithms for non-smooth optimization---such as gradient sampling, sequential quadratic programming, and exact penalty algorithms---can only guarantee the C-stationarity of the obtained solutions~(\cite{burke2005robust, curtis2012sequential, fasano2014linesearch}). Therefore, it remains to study whether the global solution of the positive RPCA is the only C-stationary point. To answer this question, we need the following two lemmas.
	\end{sloppypar}
	\begin{lemma}\label{l1_clarke}
		The following statements hold:
		\begin{itemize}
			\item[-] If $h: \mathcal{X}\rightarrow \mathbb{R}$ and $g: \mathcal{X}\rightarrow \mathbb{R}$ are continuously differentiable at $\bar{\bx}\in \mathcal{X}$, then $(h+g)^\circ(\bar{\bx},\bd) = h^\circ(\bar{\bx},\bd)+g^\circ(\bar{\bx},\bd)$ for every feasible direction $\bd$.
			\item[-] If $h: \mathcal{X}\rightarrow \mathbb{R}$ is continuously differentiable at $\bar{\bx}\in \mathcal{X}$, then $h^\circ(\bar{\bx},\bd) = h'(\bar{\bx},\bd)$ for every feasible direction $\bd$.
		\end{itemize}
	\end{lemma}
	\begin{proof}
		Refer to the textbook by~\cite{clarke1990optimization}.
	\end{proof}
	\begin{lemma}\label{l2_clarke}
		Let $h_1(\bx), h_1(\bx), ..., h_m(\bx):\mathcal{X}\rightarrow \mathbb{R}$ be continuous and locally Lipschitz functions at $\bar{\bx}\in\mathcal{X}$. Define
		\begin{equation}
		h(\bx) = \max_{1\leq i\leq m} h_i(\bx)
		\end{equation}
		and let $I(\bar{\bx})$ be the set of indices $i$ such that $h(\bar{\bx}) = h_i(\bar{\bx})$. Then,
		\begin{equation}
		h^\circ(\bar{\bx},\bd)\leq \max_{i\in I(\bar{\bx})}h_i^\circ(\bar{\bx},\bd)
		\end{equation} 
		for every feasible direction $\bd$.
	\end{lemma}
	\begin{proof}
		Consider a feasible point $\by\in \mathcal{B}(\bar{\bx},\epsilon)\cap\mathcal{X}$, where $\mathcal{B}(\bar{\bx},\epsilon)$ is the Euclidean ball with the center $\bar{\bx}$ and radius $\epsilon$. First, we prove that $I(\by)\subseteq I(\bar{\bx})$ for sufficiently small $\epsilon>0$. Notice that $h_i(\bar{\bx})<h_j(\bar{\bx})$ for every $i\in I(\bar{\bx})$ and $j\in\{1,...,m\}\backslash I(\bar{\bx})$. Therefore, due to the continuity of $h_i(\cdot)$ for every $i\in\{1,...,m\}$, it follows that there exists $\bar\epsilon>0$ such that $h_i(\by)<h_j(\by)$ for every $\by\in \mathcal{B}(\bar{\bx},\epsilon)\cap \mathcal{X}$ with $0<\epsilon<\bar{\epsilon}$. This implies that $I(\by+t\bd)\subseteq I(\by)\subseteq I(\bar{\bx})$ for every $\by\in \mathcal{B}(\bar{\bx},\epsilon)\cap \mathcal{X}$ and every feasible direction $\bd$ with sufficiently small $\epsilon>0$ and $t>0$. Now, note that
		\begin{align}
		h(\by+t\bd)-h(\by) = \max_{i\in I(\by+t\bd)} h_i(\by+t\bd)-h_i(\by)\leq \max_{i\in I(\bar{\bx})} h_i(\by+t\bd)-h_i(\by)
		\end{align}
		This implies that
		\begin{align}
		h^\circ(\bar{\bx}, \bd) = \underset{\begin{subarray}{c}
			\by\rightarrow \bx\\
			t\downarrow 0
			\end{subarray}}{\lim\sup}\frac{h(\by+t\bd)-h(\by)}{t}\leq \max_{i\in I(\bar{\bx})} \left\{\underset{\begin{subarray}{c}
			\by\rightarrow \bx\\
			t\downarrow 0
			\end{subarray}}{\lim\sup}\frac{h_i(\by+t\bd)-h_i(\by)}{t}\right\} = \max_{i\in I(\bar{\bx})} h^\circ_i(\bar{\bx}, \bd)
		\end{align}
		This completes the proof.
	\end{proof}
	
	Based on the above lemmas, we develop the following theorem.
	\begin{theorem}\label{thm14}
		Under the conditions of Theorems~\ref{thm4} and assuming that $\mG(\Omega)$ is connected, the global solution and the origin are the only C-stationary points of the symmetric positive RPCA. A similar result holds for the asymmetric positive RPCA.
	\end{theorem}
	\begin{proof}
		Without loss of generality, we only consider the symmetric case. At a given point $\bu$, the function $f(\bu)$ is locally Lipschitz and can be written as 
		\begin{equation}
		f(\bu) = \sum_{(i,j)\in\Omega}\max\{u_iu_j-X_{ij}, -u_iu_j+X_{ij}\} = \max_{\sigma\in \mathcal{M}}f_\sigma(\bu)
		\end{equation}
		where $\mathcal{M}$ is the class of functions from $\Omega$ to $\{-1,+1\}$ and $f_{\sigma}(\bu)$ is defined as 
		\begin{equation}
		f_{\sigma}(\bu) = \sum_{(i,j)\in\Omega}\sigma(i,j)(u_iu_j-X_{ij}).
		\end{equation}
		Hence,
		\begin{equation}
		f_{\mathrm{reg}}(\bu) = R(\bu)+\max_{\sigma\in \mathcal{M}}f_\sigma(\bu)
		\end{equation}
		Notice that each function $f_{\sigma}(\bu)$ is differentiable and locally Lipschitz for every $\sigma\in\mathcal{M}$. By contradiction, suppose that there exists $\bu\geq 0$ such that $\bu \not\in \left\{\bu^*, 0\right\}$ and $0\in\partial_C f_{\mathrm{reg}}(\bu)$. Furthermore, define $I(\bu)$ as the set of all functions $\sigma\in\mathcal{M}$ for which $f_{\sigma}(\bu) = f(\bu)$. Using the proof technique developed in Theorem~\ref{thm4}, one can easily verify that there exists a feasible direction $\bd$ such that $f'_{\sigma}(\bu,\bd)+R'(\bu,\bd)<0$ for every $\sigma\in I(\bu)$. By invoking Lemma~\ref{l1_clarke} for every $\sigma\in I(\bu)$, it can be concluded that $f^\circ_{\sigma}(\bu,\bd)+R^\circ(\bu,\bd)<0$. This, together with Lemma~\ref{l2_clarke}, certifies that $f_{\mathrm{reg}}^\circ(\bu,\bd)<0$, hence contradicting the assumption $0\in\partial_C f_{\mathrm{reg}}(\bu)$.
	\end{proof}
	
	
	{\section{Discussions on Extension to Rank-$r$}\label{sec:rankr}
		So far, we have characterized the conditions under which the non-negative rank-1 RPCA has no spurious local solution. However, the following question has been left unanswered: \textit{Can these results be extended to the general non-negative \textbf{rank-$\bf r$} RPCA?}
		
		As a first step toward answering this question and similar to our analysis in the rank-1 case, we consider the noiseless symmetric non-negative rank-$r$ RPCA defined as
		\begin{equation}\tag{P1-Sym-$r$}\label{p2-sym_r}
		\min_{U\in\mathbb{R}^{n\times r}_+}\quad f(U) = \|\mathcal{P}_{\Omega}(U^*{U^*}^\top-UU^\top)\|_1
		\end{equation}
		
		\noindent Indeed, a fundamental roadblock in extending the results of Section~\ref{sec3} to~\eqref{p2-sym_r} is the implicit \textit{rotational symmetry} in the solution: given a rotation matrix $R$ and a solution $\tilde{U}$ to~\eqref{p2-sym_r}, $\tilde UR$ is another feasible solution with $f(\tilde UR) = f(\tilde U)$, provided that $\tilde UR$ is a non-negative matrix. In the rank-1 case, this does not pose any problem since $R = 1$ is the only possible value. However, for the general rank-$r$ case with $r\geq 2$, this rotational symmetry undermines the strict positivity assumption of the true components. In particular, even if the true solution $U^*$ is strictly positive, there exists a rotation matrix $R$ such that $U^*R$ is non-negative with at least one zero entry. This in turn implies that Lemma~\ref{l1} and, as a consequence, the technique used in Theorem~\ref{thm1} may not be readily extended to the rank-$r$ cases. 
		
		Despite the theoretical difficulties in extending the presented results to the general rank-$r$ instances, we have indeed observed---through thousands of simulations---that in general, the sub-gradient method introduced in Section~\ref{sec:num} successfully converges to a solution $U$ that satisfies $UU^\top = U^*{U^*}^\top$, even if the measurement matrix is corrupted with a surprisingly dense noise matrix.
		To illustrate this, we consider randomly generated instances of the problem with the dimension $n = 100$ and the rank $r \in \{2,3,4,5\}$. For each instance, the elements of $U^*$ are uniformly chosen from the interval $[0.5, 2.5]$. Furthermore, each element in the upper triangular part of the noise matrix $S$ is set to $2$ and $0$ with probabilities $d$ and $1-d$, respectively. For each rank $r$ and the noise probability $d$, we consider 500 independent instances of the problem and solve them using the randomly initialized sub-gradient method. Similar to Subsection~\ref{subsec:exact}, we assume that a solution is recovered exactly if $\|UU^\top-U^*{U^*}^\top\|_F/\|U^*{U^*}^\top\|_F\leq 10^{-4}$. Figure~\ref{rankr} demonstrates the ratio of the instances for which the sub-gradient method successfully recovers the true solution. As illustrated in this figure, $d$ can be as large as $0.30$, $0.28$, $0.26$, and $0.25$ to guarantee a success rate of at least $90\%$ when $r$ is equal to $2$, $3$, $4$, and $5$, respectively. 
		
		This empirical study suggests that one of the following statements may hold for the positive rank-$r$ RPCA: (1) it is devoid of spurious local minima, or (2) its spurious local minima can be escaped efficiently using the sub-gradient method.
		Further investigation of this direction is left as an enticing challenge for future research.

		\begin{figure}
			\centering
			\includegraphics[width=.45\columnwidth]{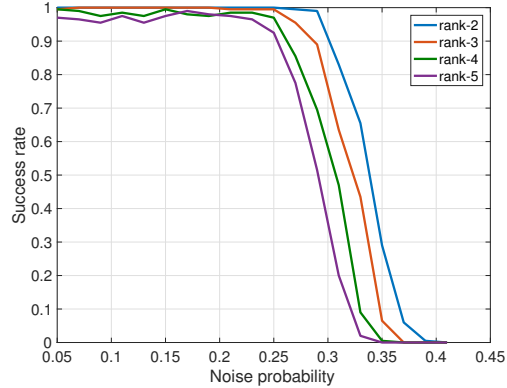}
			\caption{ \footnotesize The success rate of the randomly initialized sub-gradient method for the positive rank-$r$ RPCA.}
			\label{rankr}
		\end{figure}
		
		%
	}
	
	\section{Conclusion}
	
	This paper deals with the non-negative {rank-1} robust principal component analysis (RPCA), where the goal is to recover the true non-negative principal component of the data matrix exactly, using partial and potentially noisy measurements of the data matrix. The main difference between the RPCA and its classical counterpart is the sparse-but-arbitrarily-large values of the additive noise. The most commonly known methods for solving the RPCA are based on convex relaxations, where the problem is \textit{convexified} at the expense of significantly increasing the number of variables. In this work, we show that the original non-convex and non-smooth $\ell_1$ formulation of the positive {rank-1} RPCA problem based on the well-known Burer-Monteiro approach has benign landscape, i.e., it does not have any spurious local solution and has a unique global solution that coincides with the true components. In particular, we provide strong deterministic and statistical guarantees for the benign landscape of the positive {rank-1} RPCA and show that the absence of spurious local solutions is guaranteed to hold with a surprisingly large number of corrupted measurements. While the results on ``no spurious local minima'' are ubiquitous for smooth problems related to matrix completion and sensing, to the best of our knowledge, the results presented in this paper are the first to prove the absence of local minima when the objective function is non-smooth. {Finally, through extensive simulations, we provide strong evidence suggesting that the proposed results may hold for the general non-negative rank-$r$ RPCA. The extension of our theoretical results to this generalized problem is left as a future work.}   
	
	\section*{Acknowledgments}
	
	The authors are grateful to Javad Lavaei, Richard Zhang, and Cedric Josz for insightful discussions on earlier versions of this manuscript. Moreover, the authors thank Richard Zhang for his assistance in providing us with the code for the simulations. This work was supported by grants from ONR, AFOSR and NSF.
	
	\bibliographystyle{IEEEtran}
	\bibliography{bib.bib}
	
	\appendix
	
	\section{Proof of Lemma~\ref{l1}:}\label{app_l1}
	Without loss of generality and for simplicity, we will assume that $\mG(\Omega)$ is connected since the proof can be readily applied to each disjoint component of $\mG(\Omega)$.
	Consider a point $\bu\geq 0$ with $u_k = 0$ for some $k$. Consider $\Omega_k = \{j|(k,j)\in\Omega\} $ and note that it is non-empty due to the assumption that $\mG(\Omega)$ is connected and non-bipartite. Furthermore, if there exists $r\in\Omega_k$ such that $u_r>0$, a positive perturbation of $u_k$ will result in a feasible and negative directional derivative. Therefore, suppose that $u_r = 0$ for every $r\in\Omega_k$. Similarly, one can show that if $u_t > 0$ for some $t\in\Omega_r$ and $r\in\Omega_k$, then $\bu$ has a feasible and strictly negative directional derivative. Invoking the same argument for the neighbors of the nodes with the zero value, one can infer that $\bu = 0$. This completes the proof.~\hfill$\blacksquare$
	
	\section{Proof of Proposition~\ref{prop1}:}\label{app_prop1}
	Suppose that $\bu^*\geq 0$ and there exists an index $i$ such that $u^*_i = 0$. Without loss of generality, assume that $i = 1$ and $u^*_j>0$ for every $j\geq 2$. Next, we will show that $\bu$ defined as $u_1 = \beta>0$ and $u_j = 0$ for $j\geq 2$ is a local minimum of~\eqref{p2-sym}. Consider the perturbed version of $\bu$ as
	\begin{align}
	&\hat{u}_1 \leftarrow \beta+\epsilon_1 \\
	&\hat{u}_j\leftarrow \epsilon_j\qquad\ \ \forall j\in\{2,...,n\}
	\end{align}
	for sufficiently small $|\epsilon_1|$ and $\epsilon_2,...,\epsilon_n\geq0$. Upon defining $\Omega = \{1,...,n\}^2\backslash \{(1,1)\}$, one can write
	\begin{align}
	& f(\bu) = \sum_{j = 2}^n{u_j^*}^2+\sum_{j,k = 2, j\not=k}^nu^*_ju^*_k\\
	& f(\hat{\bu}) = \sum_{j = 2}^n{u_j^*}^2\!-\!\epsilon_j^2\!+\!\sum_{j = 2}^n(\beta+\epsilon_1)\epsilon_j+\!\!\!\!\sum_{j,k = 2, i\not=j}^n|u^*_ju^*_k-\epsilon_j\epsilon_k| \geq f(\bu)+\beta\sum_{j = 2}^n\epsilon_j\!-\!\left(\sum_{j=1}^{n}\epsilon_j\right)^2\!\!+\epsilon_1^2
	\end{align}
	It is easy to verify that there exist constants $\bar{\epsilon}_1>0$ and $\bar{\epsilon}>0$ such that for every $-\bar{\epsilon}_1\leq \epsilon_1\leq \bar{\epsilon}_1$ and $0\leq\sum_{j = 2}^n\epsilon_i\leq\bar{\epsilon}$, we have 
	\begin{equation}
	\beta\sum_{j = 2}^n\epsilon_i-\left(\sum_{i=1}^{n}\epsilon_i\right)^2+\epsilon_1^2\geq 0
	\end{equation}
	and hence $f(\hat{\bu})\geq f(\bu)$. This implies that $\bu$ is a local minimum for $f(\bu)$.~\hfill$\blacksquare$
	
	\section{Proof of Theorem~\ref{thm1_asym}:}\label{app_thm1_asym}
	First, we present a number of lemmas that are crucial to the proof of this theorem.
	
	\begin{lemma}\label{l2_asym}
		Suppose that $\mG(\bar\Omega)$ is connected and $\bw ^*>0$. Then, for every D-min-stationary point $\bw$, we have $\bw>0$ or $\bw = 0$.
	\end{lemma}
	
	\begin{proof}
		The proof is omitted due to its similarity to that of Lemma~\ref{l1}.
	\end{proof}
	\begin{lemma}\label{l8}
		Suppose that $\mG(\bar\Omega)$ is connected and $\bw^*>0$. Then, $\sum_{i=1}^{m}w_i^2=\sum_{j = m+1}^{m+n}w_j^2$ holds for every D-stationary point $\bw>0$ of~\eqref{p2_asym_sym}.
	\end{lemma}
	
	\begin{proof}
		By contradiction, suppose that $\sum_{i=1}^{m}w_i^2\not=\sum_{j = m+1}^{m+n}w_j^2$ for a D-stationary point $\bw>0$. Without loss of generality, suppose that $\sum_{i=1}^{m}w_i^2>\sum_{j = m+1}^{m+n}w_j^2$ and consider the following perturbation of $\bw$
		\begin{equation}
		\hat{w}_i\leftarrow \left\{
		\begin{array}{ll}
		w_i-w_i\epsilon &\text{if}\ 1\leq i\leq n\\
		w_i+w_i\epsilon &\text{if}\ n+1\leq i\leq n+m\\
		\end{array} 
		\right.
		\end{equation}
		For $(i,j)\in\bar{\Omega}$, one can write
		\begin{equation}
		|\hat{w}_i\hat{w}_j-\hat{w}^*_i\hat{w}^*_j| = |(w_i-w_i\epsilon)(w_j+w_j\epsilon)-\hat{w}^*_i\hat{w}^*_j| = |{w}_i{w}_j-\hat{w}^*_i\hat{w}^*_j|+w_iw_j\epsilon^2
		\end{equation}
		Therefore, we have
		\begin{equation}
		f(\hat{\bw})-f(\bw) \leq -2\alpha\left(\sum_{i=1}^{m}w_i^2-\sum_{j = m+1}^{m+n}w_j^2\right)\epsilon+O(\epsilon^2)
		\end{equation}
		This implies the existence of strictly positive and negative directional derivatives, thus resulting in a contradiction. This completes the proof.
	\end{proof}
	
	\begin{lemma}\label{l9}
		$\mG(\bar{\Omega})$ has a unique vertex partitioning.
	\end{lemma}
	
	\begin{proof}
		By contradiction, suppose that there exist two different vertex partitions $(S, T)$ and $(\bar{S}, \bar{T})$ for $\mG(\bar{\Omega})$. Since $\mG(\bar{\Omega})$ is a connected bipartite graph, $\bar{S}$ is not equal to $S$ or $T$, and therefore, $S\cap\bar{S}$ and $T\cap\bar{T}$ are not empty. Now, it is easy to observe that the nodes in $S\cap\bar{S}$ can only be connected to those in $T\cap\bar{T}$ and, similarly, the nodes in $T\cap\bar{T}$ can only be connected to those in $S\cap\bar{S}$. Therefore, unless $(S\cap\bar{S})\cup(T\cap\bar{T})$ includes all the nodes, the graph will be disconnected, contradicting our assumption. On the other hand, this implies that $S\cap\bar{S} = S$ and $T\cap\bar{T} = T$, contradicting the assumption that $(S, T)$ and $(\bar{S}, \bar{T})$ are different.
	\end{proof}
	
	\vspace{2mm}
	\noindent{\bf Proof of Theorem~\ref{thm1_asym}}
	For a D-min-stationary point $\bw$, note that if $w_i = 0$ for some index $i$, then Lemma~\ref{l2_asym} implies that $\bw = 0$, which can be easily verified to be a local maximum. We assume that $\bw^*$ satisfies $\sum_{i=1}^{m}w^{*^2}_i=\sum_{j = m+1}^{m+n}w^{*^2}_j$, which can be ensured by an appropriate scaling of $\bu^*$ and $\bv^*$ while keeping $\bu^*{\bv^*}^\top$ intact. Now, it suffices to show that for a D-stationary point $\bw>0$, we have $\bw = \bw^*$. This proves the validity of the statements of the theorem.
	
	By contradiction, suppose that $\bw>0$ with $\bw\not=\bw^*$ is a D-stationary point. In what follows, we will construct directions with strictly positive and negative directional derivatives at this point. Similar to the proof of Theorem~\ref{thm1}, one can show that
	\begin{equation}\label{eq92}
	0<\frac{w^*_1}{w_1}\leq \min_{k\in\bar\Omega_i}\frac{w^*_k}{w_k}\leq \frac{w_i}{w^*_i}\leq\max_{k\in\bar\Omega_i}\frac{w^*_k}{w_k}\leq \frac{w^*_{m+n}}{w_{m+n}}
	\end{equation}
	for every $1\leq i\leq m+n$. By contradiction, suppose that $w_i\not=w^*_i$ for some index $i$. First, note that $w^*_{m+n}/w_{m+n}>1$; otherwise, it holds that $w^*_{m+n}/w_{m+n}\leq 1$ and $w_i/w^*_i>1$, which contradict with~\eqref{eq92}. Define
	\begin{align}
	& T_1^u = \left\{i| \frac{w^*_i}{w_i} = \frac{w^*_{m+n}}{w_{m+n}}, 1\leq i\leq m \right\}, &&\hspace{-2mm} T_2^u = \left\{i| \frac{w_i}{w^*_i} = \frac{w^*_{m+n}}{w_{m+n}}, 1\leq i\leq m \right\}\nonumber\\
	& T_1^v = \left\{i| \frac{w^*_i}{w_i} = \frac{w^*_{m+n}}{w_{m+n}}, m+1\leq i\leq m+n \right\}, &&\hspace{-2mm} T_2^v = \left\{i| \frac{w_i}{w^*_i} = \frac{w^*_{m+n}}{w_{m+n}}, m+1\leq i\leq m+n \right\}
	\end{align}
	and 
	\begin{subequations}
		\begin{align}
		& N^u = \{1,\ldots, m\}\backslash(T_1^u\cup T_2^u)\\
		& N^v = \{m+1,\ldots, m+n\}\backslash(T_1^u\cup T_2^u)
		\end{align}
	\end{subequations}
	Furthermore, define $\bar{\bd}$ as 
	\begin{equation}
	\bar{d}_i= \left\{
	\begin{array}{ll}
	\frac{w_i}{w_{m+n}}-w_i\gamma &\text{if}\ i\in T_1^u\\
	-w_i\gamma &\text{if}\ i\in N^u\\
	-\frac{w_i}{w_{m+n}}-w_i\gamma &\text{if}\ i\in T_2^u\\
	\frac{w_i}{w_{m+n}}+w_i\gamma &\text{if}\ i\in T_1^v\\
	w_i\gamma &\text{if}\ i\in N^v\\
	-\frac{w_i}{w_{m+n}}+w_i\gamma &\text{if}\ i\in T_2^v
	\end{array} 
	\right.
	\end{equation}
	where 
	\begin{equation}
	\gamma = \frac{\sum_{i\in T_1^u}w_i-\sum_{i\in T_2^u}w_i-\sum_{i\in T_1^v}w_i+\sum_{i\in T_2^v}w_i}{ w_n \sum_{i=1}^{m+n}w_i}
	\end{equation}
	Similar to the symmetric case, we show that if $T_1^u\cup T_1^v$ is non-empty, then $f'(\bw,\bar\bd)<0$ and $f'(\bw,-\bar\bd)>0$, which contradicts the D-stationarity of $\bw$.
	We will only show $f'(\bw,\bar\bd)<0$ since $f'(\bw,-\bar\bd)>0$ can be proven in a similar way. Define a perturbation of $\bw$ as $\hat{\bw} = \bw+\bd\epsilon$ where $\epsilon>0$ is chosen to be sufficiently small.
	
	First, we analyze the regularization term in~\eqref{p2_asym_sym}. One can write
	\begin{align}\label{eq99}
	\left|\sum_{i=1}^{m}\hat{w}_i^2\!-\!\!\!\sum_{j = m+1}^{m+n}\hat{w}_j^2\right| \!\!\leq& \Bigg|\sum_{i=1}^{m}{w}_i^2\!-\!\!\!\sum_{j = m+1}^{m+n}{w}_j^2\nonumber\\
	&+2\left(\sum_{i\in T_1^u}\frac{w_i}{w_{m+n}}-\!\sum_{i\in T_2^u}\frac{w_i}{w_{m+n}}-\!\sum_{i\in T_1^v}\frac{w_i}{w_{m+n}}+\!\sum_{i\in T_2^v}\frac{w_i}{w_{m+n}}\right)\epsilon\nonumber\\
	&-2\gamma\left(\sum_{i=1}^{m}w_i+\sum_{i=m+1}^{m+n}w_i\right)\epsilon\Bigg|+(\frac{1}{w_n}+\gamma)^2\left(\sum_{i=1}^{m+n}w_i\right)\epsilon^2
	\end{align}
	Now, according to the definition of $\gamma$, one can easily verify that
	\begin{equation}
	2\left(\sum_{i\in T_1^u}\frac{w_i}{w_{m+n}}-\sum_{i\in T_2^u}\frac{w_i}{w_{m+n}}-\sum_{i\in T_1^v}\frac{w_i}{w_{m+n}}+\sum_{i\in T_2^v}\frac{w_i}{w_{m+n}}\right)\epsilon -2\gamma\left(\sum_{i=1}^{m}w_i+\sum_{i=m+1}^{m+n}w_i\right)\epsilon \!=\! 0
	\end{equation}
	This together with Lemma~\ref{l8}, reduces~\eqref{eq99} to
	\begin{equation}
	\left|\sum_{i=1}^{m}\hat{w}_i^2-\sum_{j = m+1}^{m+n}\hat{w}_j^2\right| \leq(\frac{1}{w_n}+\gamma)^2\left(\sum_{i=1}^{m+n}w_i\right)\epsilon^2
	\end{equation}
	To analyze the first term of~\eqref{p2_asym_sym}, similar to our previous proofs, we will divide the set $\bar{\Omega}$ into different cases (4 cases to be precise) and analyze the effect of the defined perturbation in each case. For the sake of simplicity and to streamline the presentation, we only report the final inequalities for these cases:
	\begin{enumerate}
		\item If $(i,j)\in\bar{\Omega}$ and $(i,j)\in (T_1^u\times T_1^v)\cup (T_2^u\times T_2^v)$, then
		\begin{equation}
		|\hat{w}_i\hat{w}_j-{w}^*_i{w}^*_j|\leq |w_iw_j-{w}^*_i{w}^*_j|-\frac{2w_iw_j}{w_{m+n}}\epsilon+w_iw_j\left(\frac{1}{w_{m+n}^2}-\gamma^2\right)\epsilon^2
		\end{equation}
		\item If $(i,j)\in\bar{\Omega}$ and $(i,j)\in (N^u\times (T_1^v\cup T_2^v))\cup ((T_1^u\cup T_2^u)\times N^v)$, then
		\begin{equation}
		|\hat{w}_i\hat{w}_j-{w}^*_i{w}^*_j|\leq |w_iw_j-{w}^*_i{w}^*_j|-\frac{w_iw_j}{w_{m+n}}\epsilon+w_iw_j\left(\frac{\gamma}{w_{m+n}^2}-\gamma^2\right)\epsilon^2
		\end{equation}
		\item If $(i,j)\in\bar{\Omega}$ and $(i,j)\in (T_1^u\times T_2^v)\cup (T_2^u\times T_1^v)$, then
		\begin{equation}
		|\hat{w}_i\hat{w}_j-{w}^*_i{w}^*_j|\leq |w_iw_j-{w}^*_i{w}^*_j|+w_iw_j\left(\frac{\gamma}{w_{m+n}}-\gamma\right)^2\epsilon^2
		\end{equation}
		\item If $(i,j)\in\bar{\Omega}$ and $(i,j)\in N^u\times N^v$, then
		\begin{equation}
		|\hat{w}_i\hat{w}_j-{w}^*_i{w}^*_j|\leq |w_iw_j-{w}^*_i{w}^*_j|+w_iw_j\gamma^2\epsilon^2
		\end{equation}
	\end{enumerate}
	Based on the above inequalities and due to the fact that $\mG(\bar{\Omega})$ is connected, one can easily verify that $N^u\cup N^v$ should be empty; otherwise, $\bw$ has a strictly negative (and positive) directional derivative. Based on the same reasoning, the graph induced by $T_1^u\cup T_1^v$ or $T_2^u\cup T_2^v$ should be empty. Therefore, $\mathcal{G}$ is bipartite with the components $T_1^u\cup T_1^v$ and $T_2^u\cup T_2^v$. Now, based on Lemma~\ref{l9}, $(T_1^u\cup T_1^v, T_2^u\cup T_2^v)$ induces the same vertex partitioning as $(V_u, V_v)$ (without loss of generality, assume that $T_1^u\cup T_1^v = V_u$ and $T_2^u\cup T_2^v = V_v$). This implies that
	\begin{equation}
	\frac{w_1}{w^*_1} =\cdots = \frac{w_m}{w^*_m} = \frac{w^*_{m+1}}{w_{m+1}} =\cdots=\frac{w^*_{m+n}}{w_{m+n}}>1
	\end{equation}
	Therefore,
	\begin{equation}
	\sum_{i=1}^{m}w_i > \sum_{i=1}^{m}w^*_i,\quad \sum_{i=m+1}^{m+n}w^*_i > \sum_{i=m+1}^{m+n}w_i
	\end{equation}
	Together with the assumption $\sum_{i=1}^{m}w^*_i = \sum_{i=m+1}^{m+n}w^*_i$, this implies that 
	\begin{equation}
	\sum_{i=1}^{m}w_i > \sum_{i=m+1}^{m+n}w_i
	\end{equation}
	which, according to Lemma~\ref{l8}, contradicts the D-stationarity of $\bw$. This completes the proof.~\hfill$\blacksquare$
	
	{\section{Proof of Lemma~\ref{l2}:}\label{app_l2}
		To prove this lemma, first we provide a lower bound on the probability of $\mathcal{G}(\Omega)$ being connected. Define $C_{k}$ as the number of connected components with exactly $k$ vertices in $\mathcal{G}(\Omega)$. Then, one can write:
		
		\begin{align}\label{main_connected}
		\mathbb{P}(\mathcal{G}(\Omega) \text{ is connected}) &= 1-\mathbb{P}\left(\bigcup_{k=1}^{\lceil n/2\rceil}\{C_k>0\}\right)= 1-\mathbb{P}(C_1>0) - \sum_{k=2}^{\lceil n/2\rceil}\mathbb{P}(C_k>0)
		\end{align}
		where $\lceil n/2\rceil$ denotes the smallest integer that is greater than or equal to $n/2$. Next, we provide an upper bound on $\mathbb{P}(C_k>0)$ for every $k \in \{2,\dots,\lceil n/2\rceil\}$. We have
		\begin{align}\label{Ck}
		\mathbb{P}(C_k>0)\leq \mathbb{E}\{C_k\} = \sum_{\mathcal{X}\subseteq [1:n], |\mathcal{X}| = k}\mathbb{E}\{I_{\mathcal{X}}\}
		\end{align}
		where $I_{\mathcal{X}}$ is an indicator random variable taking the value 1 if the subgraph $\mathcal{G}_{\mathcal{X}}(\Omega)$ of $\mathcal{G}(\Omega)$ induced by the set of vertices in $\mathcal{X}$ is an isolated connected component of $\mathcal{G}(\Omega)$, and it takes the value 0 otherwise. On the other hand, note that $\mathcal{G}_{\mathcal{X}}(\Omega)$ is connected if and only if it contains a spanning tree. Therefore, one can write
		\begin{align}
		\mathbb{E}\{I_{\mathcal{X}}\} &= \mathbb{P}(\mathcal{G}_{\mathcal{X}}(\Omega) \text{ has a spanning tree})\nonumber\\
		&\leq \sum_{\mathcal{T}\subset\mathcal{K}_k} \mathbb{P}(\mathcal{T} \text{ belongs to } \mathcal{G}_{\mathcal{X}}(\Omega))\nonumber\\
		&\leq k^{k-2}p^{k-1}
		\end{align}
		where $\mathcal{K}_k$ is a complete graph over $k$ vertices and $\mathcal{T}$ is a spanning tree. The last inequality is due to the fact that the number of different spanning trees in $\mathcal{K}_k$ is equal to $k^{k-2}$~(\cite{hartsfield2013pearls}). Combining the above inequality with~\eqref{Ck} results in 
		\begin{align}
		\mathbb{P}(C_k>0) &\leq {n \choose k} k^{k-2}p^{k-1} (1-p)^{k(n-k)}\nonumber\\
		&\leq \left(\frac{ne}{k}\right)^k k^{k-2} e^{-pk(n-k)}\nonumber\\
		&\leq \frac{1}{k^2}e^{-pk(n-k)+k\log n +k }\nonumber\\
		&\leq \frac{1}{k^2}e^{-\frac{k(n-1)}{2}\left(p-\frac{2\log n +2}{n-1}\right)}
		\end{align}
		where the second inequality is due to the relations ${n \choose k}\leq \left(\frac{ne}{k}\right)^k$ and $(1-p)^{k(n-k)}\leq e^{-pk(n-k)}$. Furthermore, the last inequality is due to $k\leq (n+1)/2$. Now, upon choosing $p\geq \frac{(2\eta+2)\log n +2}{n-1}$ for some $\eta>0$, one can write
		\begin{align}
		\mathbb{P}(C_k>0) \leq \frac{1}{k^2}e^{-\eta k\log n} = \frac{1}{k^2} (n^{-\eta})^k
		\end{align}
		Revisiting~\eqref{main_connected}, one can also verify that 
		\begin{align}
		\mathbb{P}(C_1>0)\leq n(1-p)^{n-1}\leq e^{-p(n-1)+\log n}\leq n^{-\eta}
		\end{align}
		provided that $p\geq \frac{(\eta+1)\log n}{n-1}$, which is implied by $p\geq \frac{(2\eta+2)\log n +2}{n-1}$. Combining this bound with~\eqref{main_connected}, one can write
		\begin{align}
		\mathbb{P}(\mathcal{G}(\Omega) \text{ is connected})&\geq 1-n^{-\eta}-\sum_{k=2}^{\lceil n/2\rceil}\frac{1}{k^2}\left(n^{-\eta}\right)^k\nonumber\\
		&\geq 1-n^{-\eta}-\frac{1}{4}\frac{n^{-2\eta}}{1-n^{-\eta}}\nonumber\\
		&\geq  1-\left(1+\frac{1}{4(n^\eta-1)}\right)n^{-\eta}\nonumber\\
		&\geq 1-\frac{5}{4}n^{-\eta}
		\end{align}
		where we have used the assumption $n\geq 2$ and $\eta\geq 1$.
		Finally, given the event that $\mathcal{G}(\Omega)$ is connected, it is non-bipartite if it has at least one self-loop. Therefore, the probability of $\mathcal{G}(\Omega)$ being non-bipartite is lower bounded by $1-(1-p)^n$. This implies that
		\begin{align}\label{con_bipartite}
		\mathbb{P}(\mathcal{G}(\Omega) \text{ is connected and non-bipartite})&\geq \left(1-\frac{5}{4}n^{-\eta}\right)\left(1-(1-p)^n\right)\nonumber\\
		&\geq\left(1-\frac{5}{4}n^{-\eta}\right)\left(1-e^{-np}\right)\nonumber\\
		&\geq\left(1-\frac{5}{4}n^{-\eta}\right)\left(1-e^{-(n-1)p}\right)\nonumber\\
		&\geq \left(1-\frac{5}{4}n^{-\eta}\right)\left(1-e^{-2}n^{-(2\eta+2)}\right)\nonumber\\
		&\geq 1-\frac{3}{2}n^{-\eta}
		\end{align}
		This completes the proof.~\hfill$\blacksquare$
	}
	
	\section{Proof of Proposition~\ref{prop4}:}\label{app_prop4}
	To prove Proposition~\ref{prop4}, we present another important result on {Erd\"os-R\'enyi} random graphs.
	
	\begin{lemma}[\cite{erdds1959random}]\label{l3}
		Assuming that $np\rightarrow 0$ as $n\rightarrow\infty$, the following properties hold with probability approaching to one:
		\begin{itemize}
			\item[-] $\mathcal{G}(n,p)$ is acyclic.
			\item[-] The size of every component of $\mathcal{G}(n,p)$ is $O(\log n)$. 
		\end{itemize}
	\end{lemma}
	
	\noindent\textbf{Proof of Proposition~\ref{prop4}:} Assuming that $np\rightarrow 0$, Lemma~\ref{l3} implies that $\mathcal{G}(\Omega)$ is the union of disjoint tree components, each with the size of at most $O(\log n)$. In what follows, we will show that, with probability approaching to one, $\mathcal{G}(\Omega)$ has at least a bipartite component without any self loops. This, together with Proposition~\ref{prop2}, will immediately conclude the proof. One can write
	\begin{align}\label{eq31}
	\mathbb{P}(\text{$\mathcal{G}(\Omega)$ has a bipartite comp.}) &\overset{(a)}{\geq} \mathbb{P}(\text{$\mathcal{G}(\Omega)$ has a tree comp. without self loops})\nonumber\\
	&\geq \mathbb{P}(\text{every comp. is a tree with size}\ O(\log n))\nonumber\\
	&\hspace{-3cm}\times \mathbb{P}(\text{no self-loop in at least one comp}|\text{every comp. is a tree with size}\ O(\log n))\nonumber\\
	&\overset{(b)}{=}\mathbb{P}(\text{every comp. is a tree with size}\ O(\log n))\nonumber\\
	&\hspace{-3cm}\times \mathbb{P}(\text{no self-loop in at least one comp}|\text{every comp. has the size}\ O(\log n))\nonumber\\
	&{\geq}\mathbb{P}(\text{every comp. is a tree with size}\ O(\log n))\nonumber\\
	&\hspace{-3cm}\times (1-\mathbb{P}(\text{every comp. has self-loops}|\text{every comp. has the size}\ O(\log n)))\nonumber\\
	&{\geq}\mathbb{P}(\underbrace{\text{every comp. is a tree with size}\ O(\log n)}_{\mathcal{A}})\nonumber\\
	&\hspace{-3cm}\times (1-\mathbb{P}(\underbrace{\text{there are at least $\Omega(n/\log n)$ self-loops}}_{\mathcal{B}}))
	\end{align}
	where $(a)$ is followed by the fact that every tree is bipartite, and $(b)$ is followed by the fact that the self-loops are included in the graph independent of other edges. {Based on Lemma~\ref{l3}, we have $\mathbb{P}(\mathcal{A})\rightarrow 1$
		as $n\rightarrow \infty$. Now, we only need to show that $\mathbb{P}(\mathcal{B})\rightarrow 0$
		as $n\rightarrow \infty$. One can easily verify that
		\begin{equation}\label{eq34}
		\mathbb{P}(\mathcal{B})\leq {n \choose \frac{n}{\log n}}p^{\frac{n}{\log n}}\leq \left(e\log n\right)^{\frac{n}{\log n}} p^{\frac{n}{\log n}}
		\end{equation}
		where the second inequality follows from the relation ${n \choose k}\leq \left(\frac{ne}{k}\right)^k$. Replacing $p = o(1/n)$ gives rise to
		\begin{align}
		\lim_{n\to\infty}\mathbb{P}(\mathcal{B})\leq \lim_{n\to\infty}(ep\log n)^{\frac{n}{\log n}} = 0
		\end{align}
		Together with~\eqref{eq31}, this implies that $\mathcal{G}(\Omega)$ will have a bipartite component without self loops with probability approaching 1.~\hfill$\blacksquare$}
	
	{\section{Proof of Lemma~\ref{l_bipartite}}\label{app_l_bipartite}
		We take an approach similar to the proof of Lemma~\ref{l2}. First, recall that $\{V_u,V_v\}$ with $V_u = \{1,\dots,m\}$ and $V_v = \{m+1,\dots,m+n\}$ is a vertex partitioning of the bipartite graph $\mathcal{G}(\bar\Omega)$. Define $C_{k,l}$ as the number of connected components with exactly $k$ vertices from $V_u$ and $l$ vertices from $V_v$. To simplify the presentation and without loss of generality, we assume that $m$ and $n$ are even. One can write:
		\begin{align}\label{main_connected_b}
		\mathbb{P}(\mathcal{G}(\bar\Omega) \text{ is connected}) &= 1-\mathbb{P}\left(\underset{\begin{subarray}{c}
			k = 0\\
			k+l\not=0
			\end{subarray}}{\bigcup^{\lceil m/2\rceil}}\bigcup_{l=1}^{\lceil n/2\rceil}\{C_{k,l}>0\}\right)\nonumber\\
		&\geq 1-\left(\mathbb{P}(C_{1,0}>0)+\mathbb{P}(C_{0,1}>0)\right) - \sum_{k=1}^{\lceil m/2\rceil}\sum_{l=1}^{\lceil n/2\rceil}\mathbb{P}\left(C_{k,l}>0\right)
		\end{align}
		First, we provide an upper bound on $\mathbb{P}\left(C_{k,l}>0\right)$ for $k = 1,\dots,\lceil m/2\rceil$ and $l = 1,\dots,\lceil n/2\rceil$. Similar to the proof of Lemma~\ref{l2}, one can write
		\begin{align}\label{Ck_b}
		\mathbb{P}(C_{k,l}>0)\leq \mathbb{E}\{C_{k,l}\} = \underset{\begin{subarray}{c}
			\mathcal{X}_u\subseteq [1:m], |\mathcal{X}_u| = k\\
			\mathcal{X}_v\subseteq [m+1:m+n], |\mathcal{X}_v| = l
			\end{subarray}}{\sum}\mathbb{E}\{I_{\mathcal{X}_u, \mathcal{X}_v}\}
		\end{align}
		where $I_{\mathcal{X}_u, \mathcal{X}_v}$ is an indicator random variable taking the value 1 if the subgraph $\mathcal{G}_{\mathcal{X}_u, \mathcal{X}_v}(\bar\Omega)$ of $\mathcal{G}(\bar{\Omega})$ induced by the set of vertices in $\mathcal{X}_u\cup\mathcal{X}_v$ is an isolated connected component of $\mathcal{G}(\bar{\Omega})$, and it takes the value 0 otherwise. On the other hand, we have
		\begin{align}
		\mathbb{E}\{I_{\mathcal{X}_u, \mathcal{X}_v}\} &= \mathbb{P}(\mathcal{G}_{\mathcal{X}_u, \mathcal{X}_v}(\bar\Omega) \text{ has a spanning tree})\nonumber\\
		&\leq \sum_{\mathcal{T}\subset\mathcal{K}_{k,l}} \mathbb{P}(\mathcal{T} \text{ belongs to } \mathcal{G}_{\mathcal{X}_u, \mathcal{X}_v}(\bar\Omega))\nonumber\\
		&\leq k^{l-1}l^{k-1}p^{k+l-1}
		\end{align}
		where $\mathcal{K}_{k,l}$ is a complete bipartite graph over two sets of vertices with the sizes $k$ and $l$, and $\mathcal{T}$ is a spanning tree. The last inequality is due to the fact that the number of different spanning trees in $\mathcal{K}_{k,l}$ is equal to $k^{l-1}l^{k-1}$~(\cite{hartsfield2013pearls}).
		Therefore, one can write
		\begin{align}\label{Ckl}
		\mathbb{P}(C_{k,l}>0)&\leq {m \choose k}{n \choose l}k^{l-1}l^{k-1}p^{k+l-1}(1-p)^{k(n-l)+l(m-k)}\nonumber\\
		&\leq \left(\frac{me}{k}\right)^k\left(\frac{ne}{l}\right)^lk^{l-1}l^{k-1}e^{-p\left(k(n-l)+l(m-k)\right)}\nonumber\\
		&\leq\frac{1}{kl}\left(\frac{k}{l}\right)^{l-k}e^{-p\left(k(n-l)+l(m-k)\right)+k\log m +l\log n+(k+l)}\nonumber\\
		&\leq\frac{1}{kl}e^{-p\left(k(n-l)+l(m-k)\right)+(k+l)(\log(mn)+1)}
		\end{align}
		where we used the relation $\left(\frac{k}{l}\right)^{l-k}\leq 1$ in the last inequality. Next, we show that the following inequality holds:
		\begin{align}\label{useful}
		k(n-l)+l(m-k)\geq (k+l)\frac{(m-1)(n-1)}{m+n}
		\end{align}
		To this goal, note that
		\begin{align}
		&k(n-l)+l(m-k)\geq (k+l)\frac{(m-1)(n-1)}{m+n}\nonumber\\
		\iff & k(m+n)(n-l)+l(m+n)(m-k)\geq (k+l)(m-1)(n-1)\nonumber\\
		\iff & (k+l)mn+kn(n-2l)+lm(m-2k)\geq (k+l)(m-1)(n-1)\nonumber\\
		\iff & kn(n-2l)+lm(m-2k)\geq -nk-ml-(n-1)l-(m-1)l
		\end{align}
		where the last inequality holds due to $l\leq (n+1)/2$ and $k\leq (m+1)/2$, which in turn implies that $kn(n-2l)+lm(m-2k)\geq -nk-ml$. Combining~\eqref{useful} and~\eqref{Ckl} leads to
		\begin{align}
		\mathbb{P}(C_{k,l}>0)\leq \frac{1}{kl}e^{-(k+l)\frac{(m-1)(n-1)}{m+n}\left(p-\frac{(m+n)(\log(mn)+1)}{(m-1)(n-1)}\right)}
		\end{align}
		Upon choosing $p\geq \frac{(m+n)((1+\eta)\log(mn)+1)}{(m-1)(n-1)}$ for some $\eta\geq 1$, one can write
		\begin{align}\label{Ckl2}
		\mathbb{P}(C_{k,l}>0)\leq\frac{1}{kl}\left((mn)^{-\eta}\right)^{(k+l)}
		\end{align}
		On the other hand, it is easy to verify that
		\begin{align}\label{C01}
		&\mathbb{P}(C_{0,1}>0)\leq n(1-p)^m\leq e^{-pm+\log n}\leq (mn)^{-\eta}\nonumber\\
		&\mathbb{P}(C_{1,0}>0)\leq m(1-p)^n\leq e^{-pn+\log m}\leq (mn)^{-\eta}
		\end{align}
		provided that $p\geq \frac{(1+\eta)\log(mn)}{m}$ and $p\geq \frac{(1+\eta)\log(mn)}{n}$, both of which are guaranteed to hold with the choice of $p\geq \frac{(m+n)((1+\eta)\log(mn)+1)}{(m-1)(n-1)}$. Combining~\eqref{C01},~\eqref{Ckl2}, and~\eqref{main_connected_b} results in
		\begin{align}
		\mathbb{P}(\mathcal{G}(\bar\Omega) \text{ is connected})&\geq 1-2(mn)^{-\eta}-\sum_{k=1}^{\lceil m/2\rceil}\sum_{l=1}^{\lceil n/2\rceil}\frac{1}{kl}\left((mn)^{-\eta}\right)^{(k+l)}\nonumber\\
		&\geq 1-2(mn)^{-\eta}-4(mn)^{-2\eta}
		\end{align}
		where we have used the assumptions $m,n\geq 2$ and $\eta\geq 1$. This completes the proof.~\hfill$\blacksquare$ 
	}
	
	\section{Proof of Lemma~\ref{l4}}\label{app_l4}
	We present the proof for the symmetric case (the proof for the asymmetric case follows directly after symmetrization and the fact that the penalty on the norm difference is zero at the positive D-stationary points). 
	First, we prove that $u_{\max}\leq 2$. It suffices to show that $u_{\max}\leq \max\{2\beta, \sqrt{2n/\lambda}\}$. This, together with the choice of $\beta$ and $\lambda$, implies $u_{\max}\leq 2$. To this goal, we only need to verify that $u_{\max}>2\beta$ implies $u_{\max}\leq\sqrt{2n/\lambda}$. 
	By contradiction, suppose that $u_{\max}>\sqrt{2n/\lambda}$. In what follows, it will be shown that $\bu$ has strictly positive and negative directional derivatives, thereby contradicting its D-stationarity. Consider a perturbation of $\bu$ as $\hat{\bu} = \bu-\mathbf{e}_{\max}\epsilon$ for a sufficiently small $\epsilon>0$, where $\mathbf{e}_{\max}$ is a vector with 1 at the location corresponding to $u_{\max}$ and 0 everywhere else. One can write
	\begin{align}\label{eqreg}
	f_{\mathrm{reg}}(\hat{\bu}) - f_{\mathrm{reg}}({\bu})&\leq \left(\sum_{i=1}^nu_i\right)\epsilon+\lambda\left(({u}_{\max}-\epsilon-\beta)^4-({u}_{\max}-\beta)^4\right)\nonumber\\
	&= \left(\sum_{i=1}^nu_i\right)\epsilon-4\lambda(u_{\max}-\beta)^3\epsilon+O(\epsilon^2)\nonumber\\
	&\overset{(a)}{\leq} \left(\sum_{i=1}^nu_i-\frac{\lambda}{2}u^3_{\max}\right)\epsilon+O(\epsilon^2)\nonumber\\
	&\leq \left(nu_{\max}-\frac{\lambda}{2}u^3_{\max}\right)\epsilon+O(\epsilon^2)
	\end{align}
	where (a) is due to the fact that $u_{\max}\geq 2\beta$ implies $u_{\max}-\beta\geq u_{\max}/2$.~\eqref{eqreg} together with $u_{\max}>\sqrt{2n/\lambda}$, implies that $-\mathbf{e}_{\max}$ is a direction with a negative directional derivative. Similarly, it can be shown that $\mathbf{e}_{\max}$ is a direction with a positive directional derivative. This contradicts the D-stationarity of $\bu$ and, hence, $u_{\max}\leq \max\{2\beta, \sqrt{2n/\lambda}\}$.
	
	Next, we aim to show that $(c/2)u^{*^2}_{\min}\leq u_{\min}$. By contradiction, suppose that there exists an index $i$ such that $(c/2)u^{*^2}_{\min}> u_i$. Now, since $u_i<1$, we have $\mathbb{I}_{u_i\geq\beta} = 0$ due to the choice of $\beta$. Consider the terms in $f_{\mathrm{reg}}(\bu)$ that involves $u_i$:
	\begin{equation}
	\sum_{j\in\Omega_i}|u_iu_j-X_{ij}| = \sum_{j\in G_i}|u_iu_j-u^*_iu^*_j| + \sum_{j\in B_i}|u_iu_j-(u^*_iu^*_j+S_{ij})|
	\end{equation}
	Considering the fact that $u_{\max}\leq 2$, one can verify the following inequality for every $(i,j)\in G$:
	\begin{equation}
	u_iu_j< cu^{*^2}_{\min}\leq u^{*^2}_{\min}\leq u_i^*u_j^*
	\end{equation}
	A similar inequality holds for $(i,j)\in B$:
	\begin{equation}
	u_iu_j< cu^{*^2}_{\min}\overset{(a)}{\leq} u^*_iu^*_j+S_{ij}
	\end{equation}
	where we have used Assumption~\ref{assum1} for $(a)$. Therefore, a positive and negative perturbation of $u_i$ results in negative and positive directional derivatives at $\bu$, thereby contradicting the D-stationarity of this point.~\hfill$\blacksquare$
	
	\section{Proof of Theorem~\ref{thm4}:}\label{app_thm4}
	The next lemma is crucial in proving Theorem~\ref{thm4}.
	\begin{lemma}\label{l5}
		Suppose that the assumptions of Theorem~\ref{thm4} hold and define
		\begin{align}
		s(\bu) = -&\underbrace{\underset{\begin{subarray}{c}
				(i,j) \in\mB\\
				i,j\in T_1
				\end{subarray}}{\sum}\frac{2u_iu_j}{u_n}+\underset{\begin{subarray}{c}
				(i,j) \in\mB\\
				i,j\in T_2
				\end{subarray}}{\sum}\frac{2u_iu_j}{u_n}+\underset{\begin{subarray}{c}
				(i,j) \in\mB\\
				i\in T_1\cup T_2, j\in N
				\end{subarray}}{\sum}\!\!\!\frac{u_iu_j}{u_n}}_{f_B(\bu)}\nonumber\\
		& +\underbrace{\underset{\begin{subarray}{c}
				(i,j) \in\mG\\
				i,j\in T_1
				\end{subarray}}{\sum}\frac{2u_iu_j}{u_n}+\underset{\begin{subarray}{c}
				(i,j) \in\mG\\
				i,j\in T_2
				\end{subarray}}{\sum}\frac{2u_iu_j}{u_n}+\underset{\begin{subarray}{c}
				(i,j) \in\mG\\
				i\in T_1\cup T_2, j\in N
				\end{subarray}}{\sum}\!\!\!\frac{u_iu_j}{u_n}}_{f_G(\bu)}+\underbrace{\sum_{i\in T_2}\frac{4u_i(u_i-1)^3}{u_n}\mathbb{I}_{u_i\geq 1}}_{f_R(\bu)}
		\end{align}
		where the sets $T_1$ and $T_2$ are defined as~\eqref{eq8} and~\eqref{eq9}, respectively. Then, for every D-stationary point $\bu>0$ such that $\bu\not=\bu^*$, the following inequalities hold with the choice of $\beta = 1$ and $\lambda = n/2$:
		\begin{itemize}
			\item[-] $f_{\mathrm{reg}}(\hat{\bu})-f_{\mathrm{reg}}(\bu)\leq -s(\bu)\epsilon+O(\epsilon^2)$ for $\hat{\bu} = \bu + \mathbf{d}\epsilon$ and a sufficiently small $\epsilon>0$.
			\item[-] $f_{\mathrm{reg}}(\hat{\bu})-f_{\mathrm{reg}}(\bu)\geq s(\bu)\epsilon-O(\epsilon^2)$ for $\hat{\bu} = \bu - \mathbf{d}\epsilon$ and a sufficiently small $\epsilon>0$.
		\end{itemize}
		where $\bd$ is defined as~\eqref{eqd}.
	\end{lemma}
	
	\begin{proof}
		To prove this lemma, first we show the validity of~\eqref{eq7}. 
		By contradiction, suppose that~\eqref{eq7} does not hold. Without loss of generality, assume that there exists an index $i$ such that $u_i/u_i^*>u_n^*/u_n$ (the case with $u_i/u_i^*<u_1^*/u_1$ can be argued in a similar way). This implies that $u_iu_j>u^*_iu^*_j$ for every $(i,j)\in\Omega$. Define $\hat{\bu} = \bu - \mathbf{e}\epsilon$ for a  sufficiently small $\epsilon>0$, where $\mathbf{e}$ is a vector with $e_k = 1$ if $k=i$ and $e_k = 0$ otherwise. One can write
		\begin{align}\label{eqreg2}
		f_{\mathrm{reg}}(\hat{\bu}) - f_{\mathrm{reg}}({\bu})&\leq -\left(\sum_{j\in G_i}u_j\right)\epsilon+\left(\sum_{j\in B_i}u_j\right)\epsilon+\lambda\left(({u}_{i}-\epsilon-1)^4-({u}_{i}-1)^4\right)\mathbb{I}_{u_i\geq 1}\nonumber\\
		&\leq -\left(\sum_{j\in G_i}u_j\right)\epsilon+\left(\sum_{j\in B_i}u_j\right)\epsilon\nonumber\\
		&\leq -\frac{cu^{*^2}_{\min}}{2}\delta(\mG(G))+2\Delta(\mG(B))
		\end{align}
		where $G_i = \{j|(i,j)\in G\}$ and $B_i = \{j|(i,j)\in B\}$. The second inequality is due to the fact that $\left(({u}_{i}-\epsilon-1)^4-({u}_{i}-1)^4\right)\mathbb{I}_{u_i\geq 1}$ is non-negative and the third inequality follows from Lemma~\ref{l4} and the definitions of $\delta(\mG(G))$, $\Delta(\mG(B))$. Based on the assumption of Theorem~\ref{thm4}, we have
		\begin{equation}
		\frac{\delta(\mG(G))}{\Delta(\mG(B))}> \frac{48}{c^2}\kappa(\bu^*)^4 = \frac{48}{c^2u^{*^4}_{\min}} >\frac{4}{cu^{*^2}_{\min}}
		\end{equation}
		which implies $(-cu^{*^2}_{\min}/2)\delta(\mG(G))+2\Delta(\mG(B))<0$, and hence, $-\mathbf{e}$ is a direction with a negative directional derivative. Similarly, it can be shown that $\mathbf{e}$ is a direction with a positive directional derivative. This contradicts the D-stationarity of $\bu$ and hence~\eqref{eq7} holds. 
		Now, we will show the correctness of the first statement. Similar to the proof of Theorem~\ref{thm1}, one can verify that
		\begin{equation}
		\sum_{(i,j)\in\Omega}|\hat{u}_i\hat{u}_j-X_{ij}| - \sum_{(i,j)\in\Omega}|{u}_i{u}_j-X_{ij}| \leq (f_B(\bu)-f_G(\bu))\epsilon+O(\epsilon^2)
		\end{equation}
		Now, we only need to bound $R(\hat{\bu})-R({\bu})$. To this goal, notice that if $i\in T_1$, then $u_i< u_i^*\leq 1$ due to the fact that $\bu\not = \bu^*$ and $u_i^*/u_i = u_n^*/u_n > 1$. Therefore, $\mathbb{I}_{u_i\geq 1} = 0$ for every $i\in T_1$. This implies that
		\begin{align}
		R(\hat{\bu})-R({\bu}) &= \sum_{i\in T_2}\left(u_i-\frac{u_i}{u_n}\epsilon-1\right)^4\mathbb{I}_{u_i\geq 1}-\sum_{i\in T_2}\left(u_i-1\right)^4\mathbb{I}_{u_i\geq 1} \nonumber\\
		&= -\sum_{i\in T_2}\frac{4u_i(u_i-1)^3}{u_n}\mathbb{I}_{u_i\geq 1}\epsilon+O(\epsilon^2)
		\end{align}
		A similar approach can be taken to prove the second statement of the lemma.
	\end{proof}
	
	\begin{lemma}\label{l1_noise}
		Suppose that $\mG(\Omega)$ has no bipartite component and every entry of $X$ is strictly positive. Then, for every D-min-stationary point $\bu$ of~\eqref{p2-sym}, we have $\bu[c] >0$ or $\bu[c] = 0$, where $\bu[c]$ is a sub-vector of $\bu$ induced by the $c^{\text{th}}$ component of $\mG(\Omega)$.
	\end{lemma}
	\begin{proof}
		The proof is similar to that of Lemma~\ref{l1}.
	\end{proof}

	
	%
	
	\noindent{\bf Proof of Theorem~\ref{thm4}:}	
	Similar to the proof of Theorem~\ref{thm1}, it suffices to show that none of the points $\bu>0$ with $\bu \not= \bu^*$ can be D-stationary. By contradiction, suppose that this is not the case, i.e., there exists a D-stationary point $\bu>0$ such that $\bu \not= \bu^*$.
	Consider the functions $f_B(\bu)$ and $f_G(\bu)$ defined in Lemma~\ref{l5}. The main idea behind the proof is to show that the term $f_G(\bu)$ always dominates $f_B(\bu)$. This, together with the non-negativity of $f_{R}(\bu)$, shows that $s(\bu)>0$ and hence, $f_{\mathrm{reg}}'(\bu, \bd)<0$ and $f_{\mathrm{reg}}'(\bu, -\bd)>0$, which is a contradiction. One can bound each term in $f_B(\bu)$ and obtain
	\begin{align}
	f_B(\bu)\!\leq&\!\frac{1}{u_n}\!\left(\!2\!\cdot\!\frac{\Delta(\mG(B))}{2}|T_1|u_{\max}^2\!+\!2\!\cdot\!\frac{\Delta(\mG(B))}{2}|T_2|u_{\max}^2\!+\!\frac{\Delta(\mG(B))}{2}(|T_1|\!+\!|T_2|)u_{\max}^2\right)\!\epsilon\!+\!O(\epsilon^2)\nonumber\\
	\leq&\frac{3}{2u_n}{\Delta(\mG(B))}(|T_1|\!+\!|T_2|)u_{\max}^2\epsilon+O(\epsilon^2)\nonumber\\
	\leq&\frac{6}{u_n}{\Delta(\mG(B))}(|T_1|\!+\!|T_2|)\epsilon+O(\epsilon^2)
	\end{align}
	where the last inequality follows from the fact that $u_{\max}\leq 2$ due to Lemma~\ref{l4}.
	Next, we derive a lower bound on $f_G(\bx)$:
	\begin{align}
	f_G(\bx)\geq& \frac{1}{u_n}\cdot\frac{\delta(\mG(G))}{2}(|T_1|+|T_2|)u_{\min}^2\epsilon+O(\epsilon^2)\nonumber\\
	\geq& \frac{1}{u_n}\cdot\frac{\delta(\mG(G))}{2}(|T_1|+|T_2|)\frac{c^2u^{*^4}_{\min}}{4}\epsilon+O(\epsilon^2)\nonumber\\
	=&\frac{c^2u^{*^4}_{\min}}{8u_n}\delta(\mG(G))(|T_1|+|T_2|)\epsilon+O(\epsilon^2)
	\end{align}
	where the first inequality is due to the fact that the minimum value for $f_G(\bu)$ happens when the neighbors of $T_1\cup T_2$ in $\mG(G)$ all belong to the set $N$ and their corresponding values in $\bu\bu^\top$ are all equal to $u_{\min}^2$. Furthermore, the second inequality is due to Lemma~\ref{l2} and the choice of $\beta$ for $R(\bu)$. 
	Therefore, one can write
	\begin{align}
	f_B(\bx)-f_G(\bx) &\leq \left(\frac{6}{u_n}{\Delta(\mG(B))}-\frac{c^2u^{*^4}_{\min}}{8u_n}\delta(\mG(G))\right)(|T_1|+|T_2|)\epsilon+O(\epsilon^2)\nonumber\\
	& = \frac{\Delta(\mG(B))c^2u^{*^4}_{\min}}{8u_n}\left(\frac{48}{c^2}\kappa(\bu^*)^4-\frac{\delta(\mG(G))}{\Delta(\mG(B))}\right)(|T_1|+|T_2|)\epsilon+O(\epsilon^2).
	\end{align}
	Therefore, the choice of $({48}/c^2)\kappa(\bu^*)^4<{\delta(\mG(G))}/{\Delta(\mG(B))}$ implies that $f_B(\bx)-f_G(\bx)<0$, thereby completing the proof.~\hfill$\blacksquare$

	
	\section{Proof of Lemma~\ref{l7}}\label{app_l7}
	{The degree of each node is equal to the summation of $n$ independent Bernoulli random variables, each with parameter $p$. Therefore, Chernoff bound yields that
		\begin{subequations}
			\begin{align}
			& \mathbb{P}(\deg(v)\geq (1+\delta) np)\leq e^{-np\delta^2/3}\\
			& \mathbb{P}(\deg(v)\leq (1-\delta) np)\leq e^{-np\delta^2/3}
			\end{align}
		\end{subequations}
		for every vertex $v$ and $0\leq\delta\leq 1$, where $\deg(v)$ is the degree of vertex $v$ in the graph. Therefore, a simple union bound leads to 
		\begin{subequations}
			\begin{align}
			& \mathbb{P}(\Delta(\mG(n,p))\geq (1+\delta) np)\leq ne^{-np\delta^2/3} = e^{-np\delta^2/3+\log n}\\
			& \mathbb{P}(\delta(\mG(n,p))\leq (1-\delta) np)\leq ne^{-np\delta^2/3} = e^{-np\delta^2/3+\log n}
			\end{align}
		\end{subequations}
		Setting $\delta = 1/2$ and assuming that $p\geq {12(1+\eta)\log n}/{n}$ for some $\eta>0$, one can write
		\begin{subequations}
			\begin{align}
			& \mathbb{P}\left(\Delta(\mG(n,p))\geq \frac{3np}{2}\right)\leq n^{-\eta}\label{eq78}\\
			& \mathbb{P}\left(\delta(\mG(n,p))\leq \frac{np}{2}\right)\leq n^{-\eta}\label{eq79}
			\end{align}
		\end{subequations}
		Furthermore, $p< {12(1+\eta)\log n}/{n}$ leads to
		\begin{align}\label{eq80}
		\mathbb{P}\left(\Delta(\mG(n,p))\geq 18(1+\eta)\log n\right)&\leq \mathbb{P}\left(\Delta\left(\mG\left(n,\frac{12(1+\eta)\log n}{n}\right)\right)\geq 18(1+\eta)\log n\right)\nonumber\\
		&\leq \mathbb{P}\left(\Delta\left(\mG\left(n,\frac{12(1+\eta)\log n}{n}\right)\right)\geq \frac{3np}{2}\right)\nonumber\\
		&\leq n^{-\eta}
		\end{align}
		Combining~\eqref{eq80} with~\eqref{eq78} and~\eqref{eq79} results in the desired inequalities.
		This completes the proof.~\hfill$\blacksquare$}
	\section{Proof of Lemma~\ref{l10}}\label{app_l10}
	
	{Define $S = \{1, ..., m\}$ and $T = \{m+1, ..., m+n\}$. Similar to the proof of Lemma~\ref{app_l7}, one can write the following concentration inequalities:
		\begin{subequations}
			\begin{align}
			& \mathbb{P}(\max_{v\in S}\{\deg(v)\}\geq (1+\delta) np)\leq me^{-np\delta^2/3}\\
			& \mathbb{P}(\min_{v\in S}\{\deg(v)\}\leq (1-\delta) np)\leq me^{-np\delta^2/3}\\
			& \mathbb{P}(\max_{v\in T}\{\deg(v)\}\geq (1+\delta) mp)\leq ne^{-mp\delta^2/3}\\
			& \mathbb{P}(\min_{v\in T}\{\deg(v)\}\leq (1-\delta) mp)\leq ne^{-mp\delta^2/3}
			\end{align}
		\end{subequations}
		which imply
		\begin{subequations}
			\begin{align}
			& \mathbb{P}(\Delta(\mG(m,n,p))\geq (1+\delta) np)\leq me^{-np\delta^2/3}+ne^{-mp\delta^2/3}\leq 2e^{-mp\delta^2/3+\log n}\\
			&\mathbb{P}(\delta(\mG(m,n,p))\leq (1-\delta) mp)\leq me^{-np\delta^2/3}+ne^{-mp\delta^2/3}\leq 2e^{-mp\delta^2/3+\log n}
			\end{align}
		\end{subequations}
		Setting $\delta = 1/2$ and assuming that $p\geq 12(1+\eta)\log n/m$ for some $\eta>0$ results in
		\begin{subequations}
			\begin{align}
			& \mathbb{P}(\Delta(\mG(m,n,p))\geq \frac{3np}{2})\leq 2n^{-\eta}\\
			&\mathbb{P}(\delta(\mG(m,n,p))\leq \frac{mp}{2})\leq2n^{-\eta}
			\end{align}
		\end{subequations}
		Furthermore, if $p< 12(1+\eta)\log n/m$, one can write
		\begin{align}\label{lowprob_asym}
		\mathbb{P}\left(\Delta(\mG(n,p))\geq \frac{18(1+\eta)n\log n}{m}\right)&\leq \mathbb{P}\left(\Delta\left(\mG\left(n,\frac{12(1+\eta)\log n}{m}\right)\right)\geq \frac{18(1+\eta)n\log n}{m}\right)\nonumber\\
		&\leq \mathbb{P}\left(\Delta\left(\mG\left(n,\frac{12(1+\eta)\log n}{m}\right)\right)\geq \frac{3np}{2}\right)\nonumber\\
		&\leq 2n^{-\eta}
		\end{align}
		This completes the proof.~\hfill$\blacksquare$}
\end{document}